\documentclass[10pt, journal, twoside]{base/IEEEtran}

\IEEEoverridecommandlockouts

\usepackage{multicol}
\usepackage[bookmarks=true]{hyperref}
\usepackage[numbers,sort&compress]{natbib}
\usepackage{amsmath}
\usepackage{amssymb}
\usepackage{amsthm}
\usepackage{bm}
\usepackage{algpseudocode}
\usepackage[ruled]{algorithm2e}
\usepackage{microtype}
\usepackage{graphicx}
\usepackage{caption}
\usepackage{subcaption}
\usepackage{booktabs}
\usepackage{xcolor}

\newcommand{\minimize}[1]{\underset{#1}{\text{minimize}}}
\newcommand{\norm}[1]{\Vert #1 \Vert}
\newcommand{\obj}{\text{obj}}
\newcommand{\pos}{\text{pos}}
\newcommand{\orn}{\text{orn}}
\newcommand{\cs}{\text{cs}}
\newcommand{\mcal}[1]{\mathcal{#1}}
\newcommand{\tp}{\mathsf{T}}
\newcommand{\SE}{\mathrm{SE}}
\newcommand{\SO}{\mathrm{SO}}

\newtheorem{remark}{Remark}
\newtheorem{proposition}{Proposition}
\newtheorem{theorem}{Theorem}

\newcommand{\insertfig}{\vspace{2ex} \includegraphics[width=\linewidth]{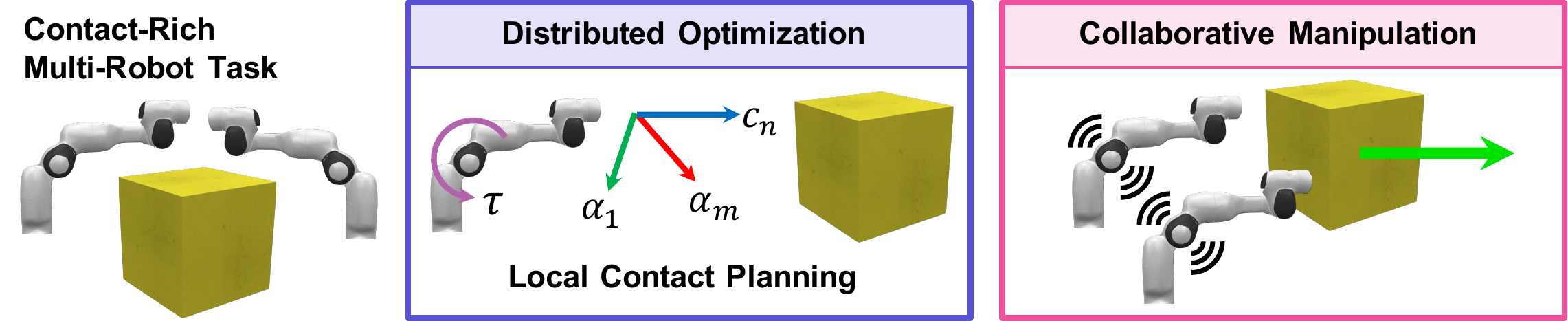}\captionof{figure}{DisCo enables multiple robots to work together, applying forces to objects and their environment through scalable, distributed contact-implicit optimization. DisCo decomposes complex centralized tasks into coupled smaller-scale problems solved efficiently by individual robots over a robot-to-robot communication network.}}\label{fig:disco}

\makeatletter
\apptocmd{\@maketitle}{\centering\insertfig}{}{}%
\makeatother

\begin{document}
    \title{DisCo: Distributed Contact-Rich Trajectory Optimization for Forceful Multi-Robot Collaboration}

    \author{~\IEEEmembership{}Ola~Shorinwa$^1$,~Matthew~Devlin$^2$,~Elliot~W.~Hawkes$^2$,~Mac~Schwager$^1$%
    \thanks{*This project was funded in part by DARPA YFA award D18AP00064 and NSF NRI awards 1830402, 1925030, and 1925373.}%
    \thanks{$^{1}$Department of Aeronautics and Astronautics, Stanford University, Stanford, CA 94305, USA {\tt\small \{shorinwa, schwager\}@stanford.edu}}%
    \thanks{$^{2}$Department of Mechanical Engineering, University of California Santa Barbara, Santa Barbara, CA, USA.
            {\tt\small \{matthewdevlin, ewhawkes\}@ucsb.edu}}%
            }

	\maketitle
	
	\begin{abstract}
		We present DisCo, a distributed algorithm for contact-rich, multi-robot tasks. DisCo is a distributed contact-implicit trajectory optimization algorithm, which allows a group of robots to optimize a time sequence of forces to objects and to their environment to accomplish tasks such as collaborative manipulation, robot team sports, and modular robot locomotion. We build our algorithm on a variant of the Alternating Direction Method of Multipliers (ADMM), where each robot computes its own contact forces and contact-switching events from a smaller single-robot, contact-implicit trajectory optimization problem, while cooperating with other robots through dual variables, enforcing constraints between robots. Each robot iterates between solving its local problem, and communicating over a wireless mesh network to enforce these consistency constraints with its neighbors, ultimately converging to a coordinated plan for the group. The local problems solved by each robot are significantly less challenging than a centralized problem with all robots' contact forces and switching events, improving the computational efficiency, while also preserving the privacy of some aspects of each robot's operation.
We demonstrate the effectiveness of our algorithm in simulations of collaborative manipulation, multi-robot team sports scenarios, and in modular robot locomotion, where \mbox{DisCo} achieves $3$x higher success rates with a $2.5$x to $5$x faster computation time. Further, we provide results of hardware experiments on a modular truss robot, with three collaborating truss nodes planning individually while working together to produce a punctuated rolling-gate motion of the composite structure. Videos are available on the project page: \href{https://disco-opt.github.io/}{https://disco-opt.github.io/.}

	\end{abstract}

	
	\section{Introduction}
\label{sec:introduction}
Many multi-robot problems, such as multi-robot manipulation, robot team sports, and modular robot locomotion, involve rich contact interactions between robots and objects or surfaces in their environment. However, these contact interactions are often ignored or abstracted away in existing distributed solution techniques (e.g., by assuming fixed, rigid grasp locations for multi-robot manipulation, or ignoring ground-robot forces in modular robot locomotion).  Contact events are inherently discrete and combinatoric in nature, and contact forces (such as friction and normal forces) can be difficult to model accurately, and quite difficult to treat computationally.  Although ignoring or abstracting away these challenges circumvents the complexity associated with contact planning, the resulting simplified model fails to exploit the rich space of contact interactions necessary for dexterous manipulation, e.g., manipulation of objects in constrained spaces, such as through doorways, windows, and hallways. We present a distributed trajectory planner amenable to a broad class of multi-robot problems involving contact, where each robot computes its local contact sequences and forces necessary to complete the task. 

Our algorithm, \emph{DisCo}, enables Distributed Contact-rich Trajectory Optimization, endowing multi-robot teams with dexterous capabilities in many multi-robot problems. For example, in collaborative manipulation, DisCo enables multi-robot teams to make and break contact with the object or environment, e.g., in narrow corridors, while manipulating the object or locomoting collaboratively with other robots. In robot team sports such as robot soccer and robot hockey, DisCo enables intelligent tactics planning, with each robot in the team working cooperatively to score a point, e.g., by passing an object (or ball) into a goal area. Moreover, in locomotion problems involving modular robots, DisCo enables each module to compute the contact interactions between each module and the environment required for locomotion, in cooperation with other modules. In all of these examples, our method can find dynamic solutions, that include free-flying ballistic phases where the robots or objects are not in static or pseudo-static equilibrium (e.g., throwing or shooting a ball, or jumping a modular robot over an obstacle). Figure~\ref{fig:intro_figure} visualizes our algorithm working with three different multi-robot scenarios.

\begin{figure*}[t]
    \centering
\includegraphics[width=\linewidth]{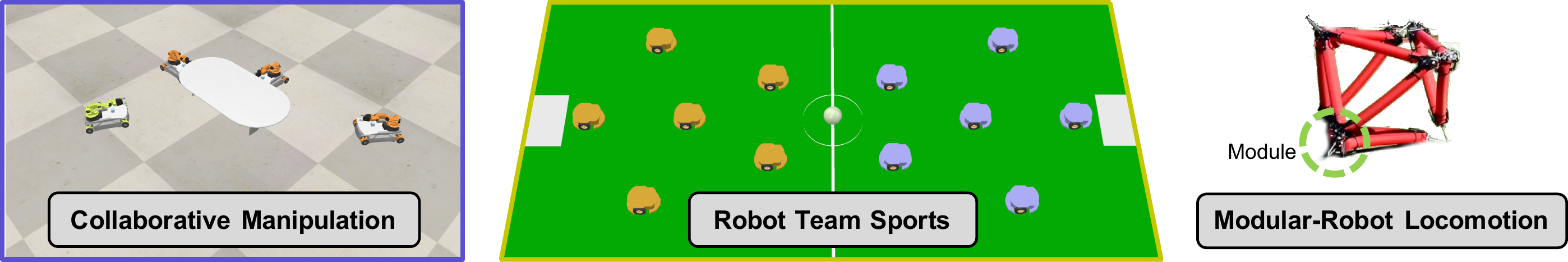}

    \caption{Our algorithm, DisCo, is amenable to a broad class of collaborative multi-robot problems involving contact, such as: (Left) collaborative manipulation, e.g., of a table; (Center) robot team sports, e.g., robot soccer; and (Right) locomotion of modular robots, e.g., \cite{usevitch_untethered_2020}. 
    In each scenario, each robot solves a smaller contact-implicit trajectory optimization problem to complete the task in collaboration with other robots. 
    }
    \label{fig:intro_figure}
\end{figure*}

We first frame multi-robot, contact-rich collaboration as a centralized contact-implicit trajectory optimization problem, in which the contact sequence and contact forces are encoded with, so-called, complementarity constraints.  This yields a large, nonconvex constrained optimization problem which is computationally difficult to solve, in general.  Our key insight is to leverage structure in this optimization problem to distributed its solution among the robots. Specifically, we use the framework of the Alternating Direction Method of Multipliers (ADMM) \cite{ola2020SOVA} to transform the centralized problem into a collection of smaller local problems (one for each robot), with coupling constraints among neighboring pairs of robots to ensure a globally consistent solution. Surprisingly, the coupling constraints only apply to a small subset of shared global variables relevant to the task, making the distributed decomposition significantly more efficient to solve than the equivalent centralized version of the problem. Applying this decomposition to the centralized problem and applying an ADMM solution technique yields our algorithm, DisCo. We emphasize that our algorithm is a trajectory planner, giving a dynamically feasible open-loop action sequence for the robots to achieve a desired collaboration objective.  In practice we re-plan as frequently as possible in an MPC outer-loop, while each robot tracks its planned trajectory with a faster low-level feedback controller. Moreover, our decomposition strategy enables trajectory optimization among heterogenous robots with different modalities. E.g., we may have pushing robots with no gripper, aerial robots towing with cables, mobile manipulators, and fixed-base manipulators all working together with the same algorithm.

This paper extends our prior work in \cite{shorinwa2021distributed}, where we consider collaborative manipulation of objects by a group of robots. In this work, we extend our algorithm to a broader class of multi-robot problems, including multi-robot team sports and modular-robot locomotion. In addition, we apply our algorithm to plan locomotion gaits in hardware for a soft modular truss robot platform, described in \cite{usevitch_untethered_2020}. 
In simulation, we show that DisCo outperforms centralized contact-implicit planners in a range of multi-robot manipulation tasks. DisCo improves the success rate of centralized through-contact planners by more than a factor of $3$, while reducing the computation time by a factor greater than $5$ in a pushing task in $\SE(2)$ and a factor greater than $2.5$ in $\SE(3)$ manipulation. Further, we demonstrate the effectiveness of our algorithm in robot team sports, where the robots attempt to get a \emph{puck} into a goal area to score a point with dynamic, ballistic passing and shooting maneuvers. Finally, we implement our algorithm on hardware, demonstrating its application to the locomotion of a soft modular truss robot platform.

Our contributions are as follows:
\begin{itemize}
    \item We present DisCo, a distributed algorithm for solving contact-rich multi-robot trajectory optimization problems, which exploits the inherently separable structure of the optimization problem to give efficient local iterates for each robot using an ADMM framework.

    \item DisCo achieves a $150\%$ to $400\%$ improvement in average solution time and more than $200\%$ higher success rate in solving multi-robot contact-implicit trajectory optimization problems compared to a centralized solver.
    
    \item We demonstrate DisCo in simulation studies of collaborative manipulation, robot team sports, and modular-robot locomotion. We also demonstrate DisCo in hardware to plan locomotion gaits for a modular soft truss robot.
\end{itemize}

Our paper is organized as follows: In Section~\ref{sec:related_work}, we survey related literature.  In Section~\ref{sec:problem_formulation}, we formulate three classes of multi-robot problems with contact---namely (i) multi-robot collaborative manipulation, (ii) multi-robot team sports tactics, and (iii) locomotion of modular robots---as contact-implicit trajectory optimization problems. Subsequently, we derive our distributed algorithm DisCo, for solving these problems, in Section \ref{sec:distributed_planning}. We present simulation results in several case studies in \ref{sec:simulations} and results of hardware experiments on locomotion gait planning for a soft modular truss robot in Section~\ref{sec:experiments}. We conclude in Section \ref{sec:conclusion}.

\section{Related Works}
\label{sec:related_work}
\subsection{Trajectory Optimization with Contact}
Trajectory optimization methods with contact planning can be classified broadly into multi-phase approaches and contact-implicit approaches, depending on the procedure taken to encode the contact dynamics in the optimization problem. Multi-phase methods require the specification of the contact phases and contact points on the robot and introduce different constraints for each phase, reflecting the robot's dynamics within each phase \cite{winkler2018gait, mordatch2012discovery, carpentier2018multicontact}. These methods do not optimize over the possible contact points on the robot. To allow for optimization over the contact points on the robot and the corresponding surfaces, some methods solve a mixed-integer convex program \cite{aceituno2017simultaneous} with additional binary variables indicating the status of a contact point and its assignment to a contact surface. However, these methods require the specification of candidate convex contact surfaces in the environment, which might not be readily available. Moreover, mixed-integer optimization presents notable computation challenges. To reduce the computation difficulty, some of these methods relax the contact dynamics constraints by introducing them into the objective function, allowing for violations of the constraints. In general, these approaches are only feasible for systems with relatively simple dynamics where the contact phases and contact points can be specified.

In contact-implicit approaches, the contact forces are incorporated as variables into the optimization problem, with an associated constraint describing the contact model. These methods introduce the contact dynamics constraints at every time instant within the optimization, and hence do not require specification of the contact phases. Many approaches represent the contact dynamics as a set of linear complementarity constraints \cite{stewart2000implicit, sleiman2019contact, posa2014direct, zhou2020accelerated}, representing the existence of contact forces only when the two surfaces have a non-zero gap. This approach for modeling the contact dynamics has been employed in simulation and trajectory optimization with promising results. However, the non-smooth contact dynamics makes numerical optimization notably challenging. Some other approaches consider a stochastic representation of the complementarity constraints to account for uncertainties in the physical properties of the colliding surfaces \cite{drnach2021robust}. 

To reduce the computation difficulty, some methods utilize a smoothed contact model with virtual forces at contact points, removing the discontinuities in the contact force, and solve the resulting problem using trust-region-based sequential convex programming \cite{onol2018comparative} or an iterative Linear-Quadratic Regulator (iLQR) approach \cite{neunert2017trajectory, carius2019trajectory}. Other methods take a bilevel optimization approach, where the upper-level problem involves computing the optimal trajectory of the problem which depend on the contact forces computed in the lower-level problem \cite{carius2018trajectory, chatzinikolaidis2020contact}. However, these approaches do not preclude the existence of contact forces at a distance and penetration of the colliding surfaces, producing unrealistic results.  None of these existing methods for trajectory optimization with contact have considered distributed solution techniques for multiple collaborating robots, as we do here.

Specifically, we employ distributed optimization for more-efficient through-contact optimization (see \cite{shorinwa2024distributedopttutorial, shorinwa2024distributedoptsurvey, halsted2021survey} for a review of distributed optimization in robotics).
We model the contact dynamics between the robots and the object (or environment) using linear complementarity constraints. By distributing the complementarity constraints among the robots, each robot solves a contact-implicit trajectory optimization problem considering only its own contact dynamics, mitigating the challenges to numerical optimization posed by the non-smooth contact constraints. With this approach, our method does not require any explicit specification of the discrete contact phases between each robot and the object.

\subsection{Collaborative Manipulation}
Now, we review relevant work in collaborative manipulation, noting that many of these works do not consider planning through contact events or contact sequences.
Notable challenges arise in centralized approaches for collaborative manipulation from the need to communicate with a central agent for computation of all forces and torques required to manipulate an object \cite{hichri2016cooperative, bais2015dynamic, ortenzi2018dual}. Addressing these challenges, distributed approaches allow each robot to compute its forces and torques without communication with a central agent. Leader-follower methods designate a single robot as a leader which computes the forces and torques for all the robots within the group. In some of these methods, the robots communicate their dynamics models and problem constraints such as the feasible forces and torques which can be applied by the robots to the leader \cite{petitti2016decentralized}. However, communication of the problem information between the leader and followers makes these methods unsuitable for problems involving varied communication topology and dynamics models. Other leader-follower methods attempt to resolve these challenges by enabling the followers to infer the motion of the leader using force or impedance sensors without communicating with the leader. The followers apply forces and torques to complement the forces applied by the leader using feedback controllers \cite{wang2016force}, impedance controllers \cite{marino2018two}, and adaptive control \cite{culbertson2018decentralized}. In these methods, the leader manipulates the object without any consideration of the limitations of the followers, possibly resulting in infeasible trajectories for the followers. Further, these approaches do not allow for trajectory planning and require the specification of a desired trajectory for the leaderby a human operator \cite{mas2012object}.

In caging approaches, the robots surround the object and move it to a desired location while keeping the object enclosed for the duration of the task using gradient-based control laws \cite{wan2012cooperative, kobayashi2012cooperative}. These approaches introduce significant limitations on the range of possible manipulation tasks as the robots must enclose the object during the task. Other methods take an optimization approach to collaborative manipulation \cite{stouraitis2020online}. In the distributed methods, each robot solves an optimization problem for its trajectory locally, enabling the robots to avoid collisions in their environments while considering the dynamics limitations of the robots \cite{verginis2018communication, ola2020collab}.

All the above methods for collaborative manipulation assume each robot grasps the object before the beginning of the task and maintains the grasp for the duration of the task. This assumption can be overly restrictive in manipulation within complex environments where the manipulation task can only be achieved through a variety of distinct contact interactions between the robots and the object. These situations require discrete contact interactions or re-grasping maneuvers between the robots and the object to manipulate the object through narrow paths and around complex features within the environment. Requiring the robot to maintain its grasp for the entirety of the task can limit the range of feasible manipulation tasks. Our method overcomes these limitations, enabling a varied set of contact interactions between the robots and the object. With our approach, each robot makes and breaks contact with the object, as required, to complete the manipulation task.

        \section{Problem Formulation}
\label{sec:problem_formulation}
\smallskip
\noindent \textbf{Notation.} Before proceeding, we define mathematical notation used in this paper. We denote the gradient of a function $f$ as $\nabla f$. We denote the matrix-weighted norm $\sqrt{x^{\tp}Mx^{\tp}}$ as $\norm{x}_{M}$, where $M$ is positive semi-definite. The Special Orthogonal Group $\SO(d)$ is the space of all rotations in dimension $d$, consisting of orthogonal matrices with a determinant of $1$, while the Special Euclidean Group $\SE(d)$ is the space of all rigid-body transformations in dimension $d$, comprising of all translations ${t \in \mathbb{R}^{d}}$ and rotation ${R \in \SO(3)}$.

\smallskip
\noindent \textbf{Communication Network.}
We represent the communication network among the robots as an undirected connected  graph $\mcal{G} = (\mcal{V}, \mcal{E})$, with the set of vertices ${\mcal{V} = \{1,\cdots,N\}}$ representing the robots, and the set of edges ${\mcal{E} \subseteq \mcal{V} \times \mcal{V}}$. An edge $(i,j)$ exists in $\mcal{E}$ if robots $i$ and $j$ share a communication link. We denote the set of robots that can communicate with robot $i$, its neighbor set, as ${\mcal{N}_{i} = \{j \mid (i, j) \in \mcal{E}\}}$.

\smallskip
\noindent \textbf{Problem}.
We pose multi-robot collaborative manipulation, tactics planning in robot team sports, and locomotion gait planning of modular robots within the framework of contact-implicit trajectory optimization, presenting a unifying formulation of these problems. We assume the contact-implicit optimization problem involves $N$ robots. In the subsequent discussion, we extend the use of the term \emph{robot} to encompass a module in a modular robot. In the following discussion, we denote the configuration of robot $i$ by ${x_{i} \in \mathbb{R}^{n_{i}}}$ and the configuration of the contact surface (which could be the surface of the pertinent object in a manipulation context or the ground surface in a locomotion problem) by ${x_{\cs} \in \mathbb{R}^{b}}$.

\subsection{Rigid-Body Dynamics}
We derive the rigid-body dynamics equations of the composite system consisting of the robots and all relevant bodies in the problem (e.g., the object in a manipulation context). We denote the configuration of the composite system  as ${\bm{x} = [x_{\cs}^{\tp}, x_{i}^{\tp}, i = 1,\cdots,N]^{\tp}}$, with ${\bm{x} \in \mathbb{R}^{n_{s}}}$; the concatenated generalized velocities as ${\dot{\bm{x}} \in \mathbb{R}^{n_{s}}}$, the concatenated torques applied by the robots as ${\bm{\tau} = [\tau_{i}^{\tp}, i = 1,\cdots,N]}$, with ${\bm{\tau} \in \mathbb{R}^{n_{\tau}}}$; and the concatenated contact forces as  ${\bm{F} = [F_{i}^{\tp}, i = 1,\cdots,N]}$, with ${\bm{F} \in \mathbb{R}^{n_{f}}}$, where $F_{i}$ denotes the contact force associated with robot $i$. The rigid-body dynamics of the composite system is described by the manipulator equation
\begin{equation}
    \label{eq:dynamics}
    \bm{M}(\bm{x})\ddot{\bm{x}} + \bm{G}(\dot{\bm{x}},\bm{x}) =  \bm{V}(\bm{x})\bm{\tau} + \bm{W}(\bm{x})\bm{F},
\end{equation}
where ${\bm{M}: \mathbb{R}^{n_{s}} \rightarrow \mathbb{R}^{n_{s} \times n_{s}}}$ represents the composite inertia, $\bm{G}(\cdot)$ represents the non-linear friction, Coriolis, and gravity terms influencing the dynamics of these bodies, and ${\bm{W}(\bm{x}) \in \mathbb{R}^{n_{s} \times n_{f}}}$ represents the Jacobian for the contact forces in the dynamics. The robots apply generalized torques (i.e., torques and forces) $\bm{\tau}$, and ${\bm{V}(\bm{x}) \in \mathbb{R}^{n_{s} \times n_{\tau}}}$ maps the generalized torques of the system $\bm{\tau}$ to the dynamics equation. The dynamics model in \eqref{eq:dynamics} applies equally to fully-actuated ($\text{rank}(\bm{V}) = n_s$) or underactuated ($\text{rank}(\bm{V}) < n_s$) robots. We assume that the dynamics constraint function in \eqref{eq:dynamics} is separable with respect to the generalized positions and velocities of the robots. This assumption is satisfied in many collaborative multi-robot problems where independently-actuated robots manipulate or interact with a common object, which is often the case in practice. For example, this assumption holds in many collaborative manipulation problems where each robot has separate end-effectors controlled independently, robot team sports and in modular-robot locomotion problems.

\subsection{Contact Dynamics}
We derive the contact dynamics between a robot and a contact surface using impulses, for better numerical stability. This approach is standard in contact-implicit optimization, (we refer interested readers to \cite{stewart2000implicit} for additional details). We denote the contact impulse associated with robot $i$ by ${f_{i} \in \mathbb{R}^{n_{f,i}}}$ consisting of a normal component $c_{i}$ and tangential components ${\alpha_{i} \in \mathbb{R}^{m}}$.
A non-zero contact impulse exists when the distance between a robot and the contact surface equals zero, which we represent as the constraint 
\begin{equation}
c_{i} \cdot \beta_{i}(x_{i},x_{\cs}) = 0
\end{equation}
where ${c_{i} \in \mathbb{R}}$ represents the magnitude of the normal contact impulse with ${c_{i} \geq 0}$ and $\beta_{i}(x_{i},x_{\cs})$ represents the distance function between the closest point on the surface of robot $i$ and the contact surface. The normal component of the contact impulse $f_{i,n}$ acts in the direction of the gradient of the distance function with respect to the configuration of the contact surface given by
\begin{equation}
\label{eq:contact_normal}
f_{i,n} = c_{i} \cdot \frac{\nabla \beta_{i}(x_{i},x_{\cs})}{\norm{\nabla \beta_{i}(x_{i},x_{\cs})}_{2}},
\end{equation}
with ${\nabla \beta_{i}(x_{i},x_{\cs}) \in \mathbb{R}^{b}}$.

The relationships between the contact impulse and the distance function between robot $i$ and the contact surface are represented by the complementarity constraints:
\begin{equation}
\label{eq:contact_complementarity}
\begin{aligned}
&c_{i} \cdot \beta_{i}(x_{i},x_{\cs}) = 0, \\
&\beta_{i}(x_{i},x_{\cs}) \geq 0, \quad c_{i} \geq 0,
\end{aligned}
\end{equation}
indicating the existence of non-zero contact impulses only when robot $i$ makes contact with the contact surface. We assume the object is manipulated by (potentially underactuated) robots, where each robot applies joint torques $\tau$ to actuate a subset of its joints.

We utilize the Coulomb friction model in formulating the contact dynamics constraints associated with the frictional impulse. We assume the tangential impulse generated by contact between the robot and the contact surface remains within the friction cone, given by the constraint $\norm{\alpha_{i}}_{2} \leq \mu c_{i}$, where $\alpha_{i}$ represents the tangential components of the contact impulse between robot $i$ and the contact surface, $\mu$ represents the coefficient of friction at the contact surface, and $c_{i}$ represents the normal component of the contact impulse. In our work, we utilize a polygonal cone approximation to the Coulomb friction cone,\footnote{One can also consider the exact friction cone, but this requires second order cone optimization tools that add complexity to the method without a significant practical benefit.  We present the polygonal cone model here for simplicity.} with $m$ tangential directions given by ${\nu_{1},\ldots,\nu_{m}}$, with ${\nu_{k} \in \mathbb{R}^{d}}$, as depicted in Figure \ref{fig:friction_cone}. With this approximation, we denote the tangential contact impulse expressed in the basis of the tangential directions by ${\alpha \in \mathbb{R}^{m}}$.  

\begin{figure}[ht]
    \centering
    \includegraphics[width=0.9\linewidth]{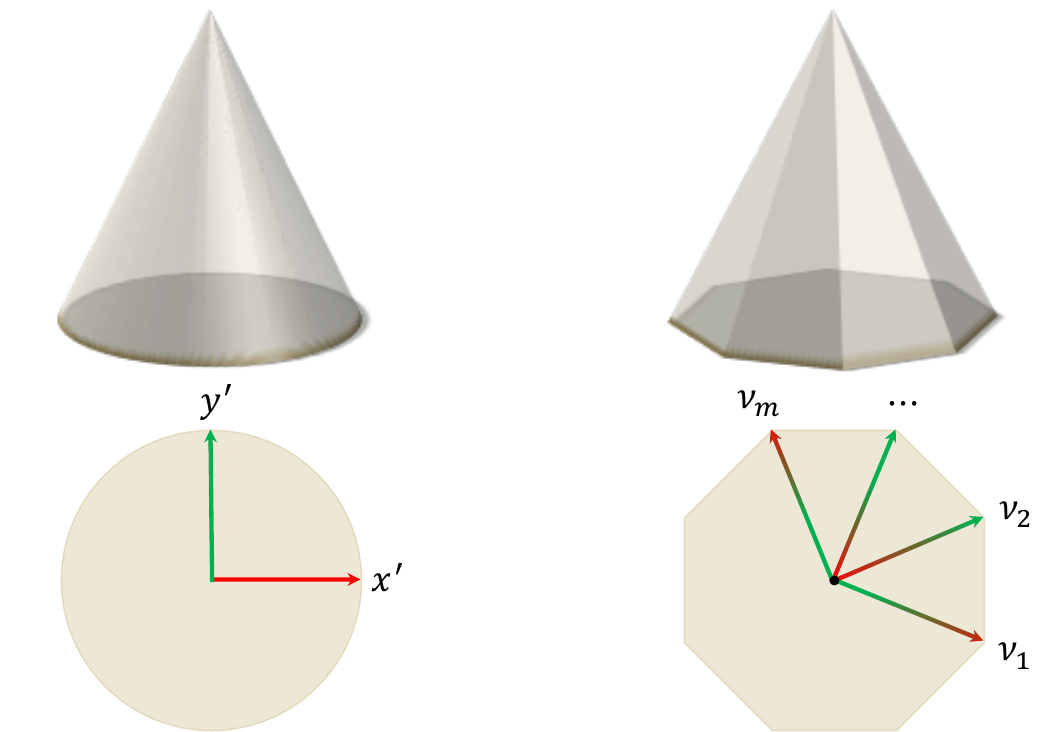}
    \caption{The Coulomb friction cone and its polygonal approximation with $m$ tangential directions denoted by $\nu_{1},\ldots,\nu_{m}$.}
    \label{fig:friction_cone}
\end{figure}

We note that the tangential contact impulse maximizes dissipation over all feasible frictional impulses. Moreover, for sliding contact interactions, the frictional contact impulse $\alpha$ acts in a direction opposite that of the sliding velocity $v$ and minimizes the following optimization problem
\begin{equation}
	\label{eq:tangential_impulse}
	\begin{aligned}
		\minimize{\alpha}\ &v^{\tp}D\alpha \\
		\text{subject to}\ & \norm{\alpha}_{1} \leq \mu c_{i} \\
									& \alpha \geq 0,
	\end{aligned}
\end{equation}
where ${D = [\nu_{1},\cdots,\nu_{m}] \in \mathbb{R}^{d \times m}}$ and ${\alpha \in \mathbb{R}^{m}}$. The optimality conditions of \eqref{eq:tangential_impulse} are given by the set of complementarity constraints
\begin{equation}
	\label{eq:contact_tangent}
	\begin{aligned}
		(\lambda e + D^{\tp}v) \perp \alpha, \\
		 \lambda e + D^{\tp}v \geq 0, \enspace \alpha \geq 0, \\
		 (\mu c_{i} - e^{\tp} \alpha) \perp \lambda, \\
		 \mu c_{i} - e^{\tp} \alpha \geq 0, \enspace \lambda \geq 0,
 	\end{aligned}
\end{equation}
where ${e \in \mathbb{R}^{m}}$ denotes the all-ones vector, and ${\lambda \in \mathbb{R}}$ denotes the Lagrange multiplier. We note that as the polygonal approximation approaches the Coulomb friction cone, the value of the Lagrange multiplier approaches the magnitude of the tangential sliding velocity.

\subsection{Collision Avoidance}
In the classes of multi-robot problems considered in this paper, we desire some contact interactions  (e.g., between the robots and an object or the ground surface to facilitate completion of the task), while we want to prohibit other contact interactions (e.g., between two robots, between robots and obstacles in the environment, and between the object being manipulated and the environment). We formulate these collision avoidance constraints as
\begin{equation}
\label{eq:collision_function}
\psi(x_{i},x_{j}) \geq \kappa_{i,j},
\end{equation}
using a collision function $\psi(x_{i},x_{j})$ specifying the distance between robots $i$ and $j$ with the minimum safe distance between the robots represented by $\kappa_{i,j}$, and similarly $\psi(x_{i}) \geq \kappa_{i}$ to prevent collision between a single robot or object and the environment.  We abuse notation slightly by using the same symbol $\psi$ for all such functions to streamline the notation.  The meaning should be clear from context.

\subsection{Optimization Problem}
We consider a contact-implicit trajectory optimization problem with the objective function of robot $i$ given by
\begin{equation}
    \int_{0}^{T} \phi_{i}(x_{i},\tau_{i},\bm{F})\ \text{dt},
\end{equation}
where $T$ denotes the duration of the task and $\phi_{i}(\cdot)$ may depend on time $t$, although we do not indicate its dependence explicitly in the problem. In many trajectory optimization problems, $\phi_{i}(\cdot)$ takes a quadratic form, including terms representing the energy expended by the actuators on the robots and the deviation of the configuration of the robots from desired values.
Summing over all robots and including all constraints, we can express the continuous-time contact-implicit trajectory optimization problem as
\begin{equation}
    \label{eq:general_problem}
    \begin{aligned}
        \minimize{\bm{x},\bm{\tau},\bm{f}}\ &\sum_{i=1}^{N} \int_{0}^{T} \phi_{i}(x_{i},\tau_{i},\bm{F})\ \text{dt} \\
        \text{subject to}\ 
        &\bm{M}(\bm{x})\ddot{\bm{x}} + \bm{g}(\dot{\bm{x}},\bm{x}) =  \bm{V}(\bm{x})\bm{\tau} + \bm{W}(\bm{x})\bm{F} \\
        &\bm{h}(\bm{x},\dot{\bm{x}}) = 0 \\
        &\bm{r}(\bm{x},\bm{\tau},\bm{F}) \leq 0 
    \end{aligned}
\end{equation}
where $\bm{h}(\cdot)$ represents constraints on the initial and desired configuration and velocities of the mobile bodies in the problem (e.g., robots) and $\bm{r}(\cdot)$ represents the contact dynamics constraints as a function of ${\bm{x},\bm{\tau},}$ and $\bm{F}$, in addition to the collision avoidance constraint in \eqref{eq:collision_function} and constraints on the torques applied by each robot. To simplify notation, we drop the subscript $t$ on $\bm{x}(t)$, $\dot{\bm{x}}(t)$, and $\ddot{\bm{x}}(t)$ in \eqref{eq:general_problem} and note that the constraints are defined over ${t \in [0, T]}$.

To compute the torques required by the robots, we discretize the continuous-time optimization problem in \eqref{eq:general_problem} for transcription to a numerical optimization problem using multiple shooting, with piecewise-constant control inputs over a time interval of $\Delta t$ seconds, which depends on the manipulation task. In line with our formulation of the contact dynamics at the impulse level, we discretize the dynamics model in \eqref{eq:dynamics} over each interval in the resulting grid, using a Backward Euler integration scheme to obtain the discrete-time dynamics model of the composite system at the impulse-momentum level,
\begin{equation}
    \label{eq:discrete_dynamics}
    \bm{m}(\dot{\bm{x}}) + \bm{g}(\dot{\bm{x}},\bm{x}) = \Delta t \cdot \bm{V}(\bm{x})\bm{\tau} + \bm{W}(\bm{x})\bm{f},
\end{equation}
where ${\bm{m}(\dot{\bm{x}}) = \bm{M}(\cdot)(\dot{\bm{x}}_{k+1} - \dot{\bm{x}}_{k})}$ and ${\bm{g}(\dot{\bm{x}},\bm{x}) = \Delta t \cdot \bm{G}(\cdot)}$ at time step $k$, with $\bm{M}$ and $\bm{G}$ evaluated at the end of the time interval. We note that ${\bm{f}}$ denotes the concatenated contact impulses of all robots.

The resulting discrete-time contact-implicit trajectory optimization problem is given by: 
\begin{equation}
    \label{eq:discrete_problem}
    \begin{aligned}
        && \minimize{\bm{x},\dot{\bm{x}},\bm{\tau},\bm{f}}\ &\sum_{i=1}^{N} \sum_{s = 0}^{T_{d}} \hat{\phi}_{i}(x_{i},\tau_{i},\bm{f},s) \\
        &&\text{subject to}\ 
        &\bm{m}(\dot{\bm{x}}) + \bm{g}(\dot{\bm{x}},\bm{x}) = \Delta t \cdot \bm{V}(\bm{x})\bm{\tau} + \bm{W}(\bm{x})\bm{f} \\
        &&&\bm{h}(\bm{x},\dot{\bm{x}}) = 0 \\
        &&&\bm{r}(\bm{x},\bm{\tau},\bm{f}) \leq 0,
    \end{aligned}
\end{equation}
where $\hat{\phi}_{i}(\cdot)$ represents the objective function evaluated numerically over the interval $[t_{s}, t_{s + 1}]$ using the corresponding values of $x_{i}$, $\tau_{i}$, and $\bm{f}$, and $T_{d}$ denotes the number of intervals in the grid. The constraint functions in \eqref{eq:discrete_problem} represent vector-valued functions applying across all time intervals. In addition, $\bm{r}(\cdot)$ represents the contact dynamics constraints described by \eqref{eq:contact_normal}, \eqref{eq:contact_complementarity}, and \eqref{eq:contact_tangent}. For notational convenience, we denote the finite-dimensional optimization variables for the generalized positions, velocities, torques, and contact impulses concatenated over all time intervals as $\bm{x}$, $\dot{\bm{x}}$, $\ddot{\bm{x}}$, $\bm{\tau}$, and $\bm{f}$, respectively.
Figure \ref{fig:contact_variables} shows a robotic manipulator in an object manipulation task, with the inertial frame shown in the lower-left corner of the figure. The contact impulses arising in the problem, initially expressed in a local reference frame attached to the end-effector of the robot, are transformed into the inertial frame in \eqref{eq:discrete_problem}. Non-zero normal and tangential contact impulses only exist when the robot is in contact with the object, as expressed by the complementarity constraints in \eqref{eq:contact_complementarity}.

\begin{figure}[ht]
    \centering
    \includegraphics[width=0.9\linewidth]{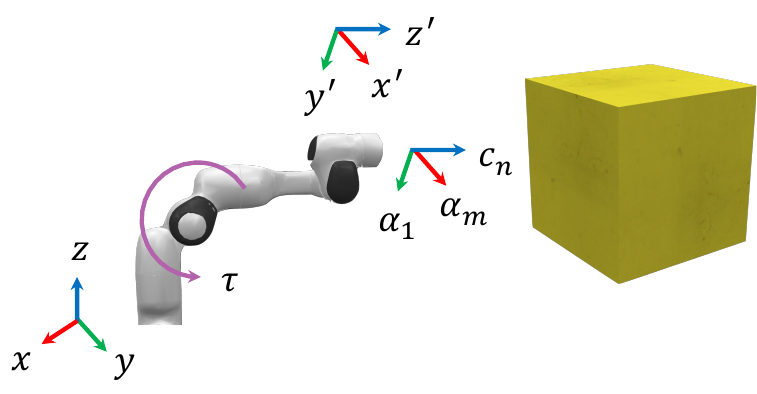}
    \caption{The contact impulses $c$ and $\alpha$, initially expressed in a local reference frame attached to the end-effector of the robot, are transformed to the inertial frame, displayed in the lower-left corner.}
    \label{fig:contact_variables}
\end{figure}

We note that some previous approaches enforce the complementarity constraint between the contact forces and the distance function in \eqref{eq:contact_complementarity} at the beginning of each time interval. This approach allows for the existence of non-zero contact forces even when contact between the robot and the object breaks away within the interval. To avoid this issue, other approaches enforce this constraint at the end of the interval, ensuring that contact forces only exist if the robot and object maintain contact at the end of the time interval. However, this approach only provides a partial remedy as non-zero contact forces can still exist within an interval when the robot does not remain in contact with the object as long as the robot contacts the object at the end of the interval. To preclude this situation, we introduce a constraint on the contact impulse and the relative velocity between robot $i$ and the object $\dot{\beta}_{i}(x_{i},x_{\obj})$ given by
\begin{equation}
	\label{eq:contact_sustained}
	c_{i} \cdot (\dot{\beta}_{i}(t + \Delta t) + \theta \dot{\beta}_{i}(t)) = 0,
\end{equation}
where $\theta \in \mathbb{R}$ represents the coefficient of restitution. For the manipulation tasks in this work, we focus on inelastic collisions between the robot and the object, described by the coefficient of restitution ${\theta = 0}$. This constraint ensures that non-zero contact forces exist between the robot and the object only when the robot remains in contact with the object over the time interval, indicated by zero relative velocity between the robot and the object over this time interval.
In Section \ref{sec:experiments}, we provide concrete examples of the general contact-implicit trajectory optimization problems given by \eqref{eq:general_problem} and \eqref{eq:discrete_problem} in each class of multi-robot problems considered in this work.

        \section{Distributed Trajectory Optimization}
\label{sec:distributed_planning}

We utilize a receding-horizon approach \cite{garcia1989model, schwenzer2021review, nascimento2018nonholonomic, shorinwa2023distributedmpc} in solving the numerical optimization problem in \eqref{eq:discrete_problem}, where we solve the problem over a smaller time span of duration $T_{r}$. The robots apply the first set of torques from the resulting solution. Subsequently, we advance the time span by one interval which defines the next optimization problem. We repeat this process until the completion of the manipulation task, with each numerical optimization problem given by
\begin{equation}
    \label{eq:discrete_problem_mpc}
    \begin{aligned}
        \mcal{P}_{C}: && \minimize{\bm{x},\dot{\bm{x}},\bm{\tau},\bm{f}}\ &\sum_{i=1}^{N}  \ell_{i}(x_{i},\tau_{i},\bm{f}) \\
        &&\text{subject to}\ 
        &\bm{m}(\dot{\bm{x}}) + \bm{g}(\dot{\bm{x}},\bm{x}) =  \bm{V}(\bm{x})\bm{\tau} + \bm{W}(\bm{x})\bm{f} \\
        &&&\bm{h}(\bm{x},\dot{\bm{x}}) = 0 \\
        &&&\bm{r}(\bm{x},\bm{\tau},\bm{f}) \leq 0,
    \end{aligned}
\end{equation}
with
\begin{equation}
\ell_{i}(x_{i},\tau_{i},\bm{f}) = \sum_{s = 0}^{T_{r}} \phi_{i}(x_{i},\tau_{i},\bm{f},s).
\end{equation}
The problem $\mcal{P}_{C}$ consists of the objective function and local constraints of all robots, which is unavailable to any single robot, coupling the computation of the required torques of each robot. Utilizing a centralized approach would be inefficient, generally, as noted earlier. To overcome this challenge, we leverage distributed optimization. In particular, we utilize ADMM \cite{mateos2010distributed, ola2020SOVA}, in contrast to distributed first-order methods \cite{nedic2017achieving, shi2015extra, xi2017add} and distributed Newton or quasi-Newton methods \cite{eisen2017decentralized, liu2023communication, shorinwa2024distributeddqn}.

We derive DisCo, our distributed algorithm for contact-rich trajectory optimization problems. In DisCo, each robot computes its torques from \eqref{eq:discrete_problem_mpc} without computing the torques of other robots, while collaborating with other robots to complete a specified task. To ease exposition of our algorithm, we present our algorithm in a setting where the robot and the contact body (object) is non-stationary, e.g., in collaborative manipulation and robot team sports. However, we note that the resulting algorithm applies to other collaborative multi-robot problems involving contact. We can decouple the composite rigid-body dynamics of the robots and contact body, expressing the dynamics constraint in the form:
\begin{equation}
	\label{eq:distributed_dynamics}
	\begin{aligned}
		m_{i}(\dot{x}_{i}, \dot{x}_{\cs})& + g_{i}(\dot{x}_{i},x_{i},\dot{x}_{\cs},x_{\cs}) =  \\ &V_{i}(x_i, x_{\cs})\tau_{i} + W_{i}(x_i, x_{\cs})\bm{f},
	\end{aligned}
\end{equation}
where $m_{i}(\cdot)$ represents the inertia terms and $g_{i}(\cdot)$ represents the non-linear friction, Coriolis, and gravity terms for the dynamics constraints involving robot $i$, similar to \eqref{eq:discrete_problem_mpc}. The dynamics constraint in \eqref{eq:distributed_dynamics} highlights the form of the coupling between the generalized positions and velocities of each robot. We note, however, that the resulting individual dynamics constraints depend on the contact body's generalized positions and velocities. As a result, we refer to the contact body's generalized positions and velocities as coupling variables.

For distributed computation of the torques in \eqref{eq:discrete_problem}, we introduce local copies of the optimization variables for the contact forces applied by the robots and the contact body's configuration and velocity, resulting in the distributed optimization problem
\begin{equation}
\label{eq:distributed_problem}
\begin{aligned}
\mcal{P}_{D}: &&\minimize{\breve{\bm{x}},\dot{\breve{\bm{x}}},\bm{\tau},\breve{\bm{f}}}\ &\sum_{i=1}^{N} \ell_{i}(x_{i},\tau_{i},\bm{f}_{i}) \\
&&\text{subject to}\ 
&\bm{v}_{i}(\breve{x}_{i},\dot{\breve{x}}_{i},\tau_{i},\bm{f}_{i}) = 0 \quad \forall i \in \mcal{V} \\
&&&h_{i}(\breve{x}_{i},\dot{\breve{x}}_{i}) = 0 \quad \forall i \in \mcal{V}  \\
&&&r_{i}(\breve{x}_{i},\dot{\breve{x}}_{i},\tau_{i},f_{i}) \leq 0\quad \forall i \in \mcal{V}\\
&&&\bm{f}_{i} = \bm{f}_{j} \quad \forall j \in \mcal{N}_{i},\ \forall i \in \mcal{V} \\
&&&\mathring{x}_{i} = \mathring{x}_{j} \quad \forall j \in \mcal{N}_{i},\ \forall i \in \mcal{V}
\end{aligned}
\end{equation}
where ${\breve{x}_{i} = [x_{i}^{\tp}, x_{\cs,i}^{\tp}]^{\tp}}$ represents the configuration of robot~$i$ and its local copy of the contact body's configuration concatenated over all times steps, ${\dot{\breve{x}}_{i} = [\dot{x}_{i}^{\tp}, \dot{x}_{\obj,i}^{\tp}]^{\tp}}$ represents the corresponding velocities, ${\breve{\bm{x}} = [\breve{x}_{i}^{\tp},\ i = 1,\cdots,N]^{\tp}}$, ${\dot{\breve{\bm{x}}} = [\dot{\breve{x}}_{i}^{\tp},\ i = 1,\cdots,N]^{\tp}}$, and ${\breve{\bm{f}} = [\bm{f}_{i}^{\tp},\ i = 1,\cdots,N]^{\tp}}$ represents the concatenation of the local optimization variables of the contact body's configuration and velocities and the contact forces applied by the robots. Further, we denote robot $i$'s local variables of the contact body's configuration and velocity as ${\mathring{x}_{i} = [x_{\cs,i}^{\tp},\dot{x}_{\cs,i}^{\tp}]^{\tp}}$, with ${\mathring{x}_{i} \in \mathbb{R}^{n_{o}}}$. The constraint function $\bm{v}_{i}(\cdot)$ represents the dynamics constraints in \eqref{eq:distributed_dynamics} as a function of robot $i$'s copy of the optimization variables associated with the contact body's configuration and velocity. 

\begin{proposition}
    The optimization problem $\mcal{P}_{D}$ in \eqref{eq:distributed_problem} is equivalent to the optimization problem $\mcal{P}_{C}$ in \eqref{eq:discrete_problem_mpc} with the same optimal solution and optimal objective value.
\end{proposition}

\begin{proof}
    All robots compute the same composite contact force from the equality constraints between $\bm{f}_{i}$ and $\bm{f}_{j}$ ${\forall (i,j) \in \mcal{E}}$ in \eqref{eq:distributed_problem}, noting that the communication graph $\mcal{G}$ is connected. Likewise, all robots compute the same trajectory for the contact body from the equality constraints on $\mathring{x}$. Consequently, we can replace the local composite contact force along with the configuration and velocity of the contact body computed by each robot by common optimization variables $\bm{\tilde{f}}$, $\tilde{x}_{\cs}$, and $\dot{\tilde{x}}_{\cs}$, respectively. With these variables, the optimization problem in \eqref{eq:distributed_problem} has the same feasible set with the problem in \eqref{eq:discrete_problem_mpc}, in addition to having the same objective function. As such, the optimization problems in \eqref{eq:distributed_problem} and \eqref{eq:discrete_problem_mpc} have the same optimal solution and optimal objective value.
\end{proof}

\begin{remark}
    In formulating the distributed trajectory optimization problem in \eqref{eq:distributed_problem}, we assumed separability of the dynamics constraint in \eqref{eq:discrete_problem_mpc} with respect to the robots' generalized positions and velocities. This assumption simplifies the discussion of our algorithm. In general, our algorithm is not reliant on this assumption. When this assumption fails to hold, the distributed trajectory optimization problem takes a slightly-different form. In particular, each robot would have to keep a local copy of each coupling variable, with additional consistency (equality) constraints between the corresponding optimization variables of the robot and its neighbors. 
\end{remark}

To maximize the computational efficiency of our approach, we apply the SOVA optimization method in \cite{ola2020SOVA} to derive DisCo, a distributed method for \eqref{eq:distributed_problem}. Using SOVA, each robot only computes variables relevant to its local objective and constraint functions rather than computing the entire set of optimization variables, as done in other distributed approaches. Our approach proves particularly efficient in solving the problem in \eqref{eq:distributed_problem}, considering the computation challenges introduced by the non-smooth contact dynamics constraints.
We distribute the problem variables among the robots based on the relevance of each variable to each robot. For example, each robot has no use of the torques applied by the other robots, rendering optimization over these variables unnecessary for each robot. Hence, robot $i$ computes its configuration and the contact body's configuration $\breve{x}_{i}$, the corresponding velocities $\dot{\breve{x}}_{i}$, torques $\tau_{i}$, and its local contact forces $\bm{f}_{i}$.

Upon distributing the problem variables among the robots, we derive update procedures for the optimization variables of each robot from the augmented Lagrangian of \eqref{eq:distributed_problem}. In deriving DisCo, we modify the SOVA algorithm to improve the efficiency of the iterative, numerical optimization procedures. We introduce a positive-definite matrix-valued penalty parameter in the augmented term arising in the augmented Lagrangian. The corresponding augmented Lagrangian $\mcal{L}_{a}$ is given by
\begin{equation}
\label{eq:lagranigan}
\begin{aligned}
\mcal{L}_{a}(\cdot) &= \sum_{i=1}^{N} \Big(\ell_{i}(x_{i},\tau_{i},\bm{f}_{i}) \\
& \hspace{3em} + \sum_{j \in \mcal{N}_{i}} \phi_{ij}^{\tp}(\bm{f}_{i} - a_{ij}) + \nu_{ij}^{\tp}(\bm{f}_{j} - b_{ij}) \\
& \hspace{3em} + \sum_{j \in \mcal{N}_{i}} \gamma_{ij}^{\tp}(\mathring{x}_{i} - u_{ij}) + \eta_{ij}^{\tp}(\mathring{x}_{j} - w_{ij}) \\
& \hspace{3em} + \frac{1}{2} \sum_{j \in \mcal{N}_{i}} \norm{\bm{f}_{i} - a_{ij}}_{W_{f}}^{2} + \norm{\bm{f}_{j} - b_{ij}}_{W_{f}}^{2} \\
& \hspace{3em} + \frac{1}{2} \sum_{j \in \mcal{N}_{i}} \norm{\mathring{x}_{i} - u_{ij}}_{W_{x}}^{2} + \norm{\mathring{x}_{j} - w_{ij}}_{W_{x}}^{2} \Big)
\end{aligned}
\end{equation}
with the dual variables ${\phi_{ij} \in \mathbb{R}^{n_{f}}}$, ${\nu_{ij} \in \mathbb{R}^{n_{f}}}$, ${\gamma_{ij} \in \mathbb{R}^{n_{o}}}$, and ${\eta_{ij} \in \mathbb{R}^{n_{o}}}$ for the equality constraints between the contact forces and the contact body's configuration and velocity of neighboring robots. In addition, ${W_{f} \in \mathbb{R}^{n_{f} \times n_{f}}}$ and  ${W_{x} \in \mathbb{R}^{n_{o} \times n_{o}}}$ denote positive-definite penalty parameters associated with the contact force and the contact body's configuration variables, respectively. Each robot updates its primal optimization variables as the minimizers of the augmented Lagrangian using its dual variables at the previous iteration, before subsequently updating its dual variables via gradient ascent. Robot $i$ solves the local optimization problem
\begin{equation}
\label{eq:primal_update}
\begin{aligned}
\mcal{P}_{L}: &&\minimize{\breve{x}_{i},\dot{\breve{x}}_{i},\tau_{i},\bm{f}_{i}}\ & \mathcal{J}_{i}(x_{i},\tau_{i},\bm{f}_{i})  \\
&&\text{subject to}\ 
&\bm{v}_{i}(\breve{x}_{i},\dot{\breve{x}}_{i},\tau_{i},\bm{f}_{i}) = 0 \\
&&&h_{i}(\breve{x}_{i},\dot{\breve{x}}_{i}) = 0  \\
&&&r_{i}(\breve{x}_{i},\dot{\breve{x}}_{i},\tau_{i},f_{i}) \leq 0
\end{aligned}
\end{equation}
with
\begin{equation}
\begin{aligned}
\mathcal{J}_{i}(\cdot) &= \ell_{i}(x_{i},\tau_{i},\bm{f}_{i}) + p_{i}^{k\tp}\bm{f}_{i} + q_{i}^{k\tp}\mathring{x}_{i} \\
& \hspace{0.35em} + \sum_{j \in \mcal{N}_{i}} \Big\Vert{\bm{f}_{i} - \frac{\bm{f}_{i}^{k} + \bm{f}_{j}^{k}}{2}}\Big\Vert_{W_{f}}^{2} + \Big\Vert{\mathring{x}_{i} - \frac{\mathring{x}_{i}^{k} + \mathring{x}_{j}^{k}}{2}}\Big\Vert_{W_{x}}^{2}\ ,
\end{aligned}
\end{equation}
to update its problem variables.

Notably, the local contact-implicit trajectory optimization problem ($\mcal{P}_{L}$) solved by each robot in \eqref{eq:primal_update} consists of a \emph{single} set of complementarity constraints for the non-smooth contact dynamics of the robot, reducing the numerical challenges to solving the global optimization problem ($\mcal{P}_{C}$) in \eqref{eq:discrete_problem_mpc}. This approach enables the robots to compute a solution to \eqref{eq:primal_update} even in cases when solving \eqref{eq:discrete_problem_mpc} proves particularly difficult. Moreover, each robot does not require the dynamics models, objective functions, and constraints of other robots in \eqref{eq:primal_update}. Since, in general, the optimization variable computed by each robot in \eqref{eq:primal_update} does not change as the number of robots increases, our method scales efficiently to multi-robot problems with large groups of collaborating robots.

After computing its problem variables, robot $i$ shares its local contact body's trajectory $\mathring{x}_{i}^{k+1}$ and contact force $\bm{f}_{i}^{k+1}$ with its neighbors and updates its dual variables ${p_{i} \in \mathbb{R}^{n_{f}}}$ and ${q_{i}\in \mathbb{R}^{n_{o}}}$ using
\begin{equation}
\label{eq:dual_update}
\begin{aligned}
p_{i}^{k+1} &= p_{i}^{k} + W_{f} \sum_{j \in \mcal{N}_{i}} \big(\bm{f}_{i}^{k+1} - \bm{f}_{j}^{k+1}\big) \\
q_{i}^{k+1} &= q_{i}^{k} + W_{x} \sum_{j \in \mcal{N}_{i}} \big(\mathring{x}_{i}^{k+1} - \mathring{x}_{j}^{k+1}\big),
\end{aligned}
\end{equation}
where the dual variables $p_{i}$ and $q_{i}$ represent composite dual variables for $\phi_{ij}$, $\nu_{ij}$, $\gamma_{ij}$, and $\eta_{ij}$. The update procedure in \eqref{eq:dual_update} results from simplification of the dual update procedure described in \cite{ola2020SOVA}.

\begin{remark}
	Although each robot considers only its local constraints in \eqref{eq:primal_update}, the resulting trajectories of the robots and contact body satisfy all the problem constraints in \eqref{eq:distributed_problem} since all the robots compute the same trajectory for the contact body and all the problem constraints are enforced by at least one robot.
\end{remark}

\begin{theorem}
	Assuming the objective function of \eqref{eq:distributed_problem} has Lipschitz continuous gradients, the trajectories of the robots and the contact body converge to a locally optimal solution of \eqref{eq:distributed_problem}.
\end{theorem}

\begin{proof}
 The theorem follows mostly from \cite{wang2019global}. Here, we briefly highlight the salient points. If the objective function of \eqref{eq:distributed_problem} has Lipschitz continuous gradients, the augmented Lagrangian $\mcal{L}_{a}(\cdot)$ of \eqref{eq:lagranigan} decreases monotonically at every iteration until convergence \cite{wang2019global}. In addition, the sequence ${\{\bm{x}^{k},\dot{\bm{x}}^{k},\bm{\tau}^{k},\bm{f}^{k},p^{k},q^{k}\}}$ converges to a stationary point ${\{\bm{x}^{\star},\dot{\bm{x}}^{\star},\bm{\tau}^{\star},\bm{f}^{\star},p^{\star},q^{\star}\}}$ of $\mcal{L}_{a}(\cdot)$, corresponding to a locally optimal solution of \eqref{eq:distributed_problem}. 
\end{proof}

\begin{algorithm}[th]
	\caption{Distributed Contact-Rich Trajectory Optimization (DisCo)}
	\label{alg:distributed_planning}
	
	\SetKwRepeat{doparallel}{do in parallel}{while}
	
	\doparallel( $i = 1,\cdots,N$) {manipulation task is in progress}{
		$(\breve{x}_{i},\dot{\breve{x}}_{i},\tau_{i},\bm{f}_{i}) \leftarrow$ \textit{OptimizeTrajectory}$(\breve{x}_{i}(t),\dot{\breve{x}}_{i}(t))$ \\
		$\tau_{i}(t) \leftarrow \tau_{i}(0)$ \\
		Apply torque $\tau_{i}(t)$.
	}
	
\end{algorithm}

\begin{algorithm} [th]
    \caption{OptimizeTrajectory($\breve{x}_{i}\text{(}t\text{)},\dot{\breve{x}}_{i}\text{(}t\text{)}$)}
    \label{alg:distributed_update}
	
	\SetKwRepeat{doparallel}{do in parallel}{while}
	
	\textbf{Initialization:} \\
	{\addtolength\leftskip{0.8em}
		$k \leftarrow 0$  \\
		$(p_{i},q_{i})^{0} \leftarrow (0, 0)$ \\[0.35em]
		$(\breve{x}_{i},\dot{\breve{x}}_{i},\tau_{i},\bm{f}_{i})^{0} \leftarrow $ $\underset{\breve{x}_{i},\dot{\breve{x}}_{i},\tau_{i},\bm{f}_{i}}{\mathrm{argminimize}}\big\{$Problem \eqref{eq:primal_update}$\big\}$ \\
	}
	
	\doparallel( $i = 1,\cdots,N$){not converged or stopping criterion is not met}{
		$(\breve{x}_{i},\dot{\breve{x}}_{i},\tau_{i},\bm{f}_{i})^{k+1} \leftarrow $ Problem \eqref{eq:primal_update} \\[0.4em]
		\emph{Communication Procedure}:\\
		{\addtolength\leftskip{0.8em}
			Robot $i$ shares $(\mathring{x}_{i}, \bm{f}_{i})$ with its  neighbors. \\
		}
		$(p_{i},q_{i})^{k+1} \leftarrow $ Procedure \eqref{eq:dual_update} \\[0.4em]
		$k \leftarrow k + 1$		
	}
	\KwRet{$(\breve{x}_{i},\dot{\breve{x}}_{i},\tau_{i},\bm{f}_{i})$}
	
\end{algorithm}

Algorithm \ref{alg:distributed_planning} outlines DisCo for distributed contact-rich trajectory optimization problems.
The robots compute their torques from the \mbox{\textit{OptimizeTrajectory}} procedure for dexterous interaction with the contact body at each time instant, repeating the procedure during the manipulation task, over a receding horizon. Each robot does not share its trajectory and torques with other robots, which reduces the communication bandwidth required for implementation of our method. We warmstart each subsequent instance of the trajectory optimization problem using the optimal solution from the previous instance of the optimization problem, which improves the convergence of the solution, especially in situations where the optimization problem changes slowly over time.
        \section{Simulations}
\label{sec:simulations}

We evaluate our distributed contact-implicit trajectory optimization method DisCo in simulation across a range of multi-robot problems, including collaborative multi-robot object manipulation and tactics planning in robot team games. In the collaborative manipulation setting, we compare the performance of DisCo to centralized contact-implicit trajectory optimization methods \cite{sleiman2019contact, posa2014direct}. We solve the resulting contact-implicit optimization problems on a consumer-grade laptop with an Intel i7 processor using IPOPT \cite{wachter2006implementation}, an interior-point optimization solver, and set the maximum number of iterations at $5000$. In each optimization problem, we use the MA-$57$ linear solver within IPOPT. 

\subsection{Multi-Robot Collaborative Manipulation}
We consider a manipulation problem where $N$ robots manipulate an object to a desired position and orientation, specified by a transformation ${\mathbb{T} \in \SE(d)}$, with ${d = 2}$ or ${d = 3}$. The robots begin from arbitrary locations at the beginning of the manipulation task and apply contact forces on the object to complete the task. We assume the inertial properties of the object and robots are known, in addition to the coefficient of friction. Consequently, we can formulate the multi-robot collaborative manipulation problem in the form given in \eqref{eq:discrete_problem}. We specify constraints on the initial and desired configurations of the robots and the object in the function $\bm{h}(\cdot)$. Further, the distance function $\beta_{i}(\cdot)$ specifies the distance between robot $i$ and the object. In general, the objective function includes a penalty on the deviation of the configuration of the object from its desired configuration, and in general, the objective function is separable, in the sense that it can be expressed as a sum of individual components associated with each robot. Our formulation also includes the specification of collision avoidance constraints, which is often not considered by other algorithms for multi-robot collaborative manipulation in which robots are rigidly attached to the object.

\subsubsection{$\mathrm{SE}(2)$ Manipulation} %
\label{sec:rod_manipulation}
In this problem, we consider a manipulation task with $N$ robots where the robots collaboratively manipulate a rod through contact, by sliding it along a surface with friction. We assume that the coefficient of friction is known. We represent the configuration of the rod as ${x_{\obj} \in \mathrm{SE}(2)}$, parameterized by the rod's position ${x_{\obj,\pos} \in \mathbb{R}^{2}}$ and orientation ${x_{\obj,\orn} \in \mathbb{R}}$, with respect to an inertial frame. Further, we denote the velocity of the object as ${\dot{x}_{\obj} \in \mathbb{R}^{3}}$ (defined with respect to the parameters of its configuration) and the configuration of robot $i$ as ${x_{i} \in \mathbb{R}^{2}}$. The robots manipulate the rod from its initial configuration to a desired position and orientation, given by the transformation ${\mathbb{T}_{\mathrm{des}} \in \SE(2)}$, before bringing the rod to rest, with ${\dot{x}_{\obj}(T) = 0}$ at the end of the manipulation task, where $T$ denotes the duration of the task.  Figure \ref{fig:manip_rod} illustrates the manipulation task performed by four robots.

\begin{figure*}[th]
    \centering
    \begin{subfigure}[t]{.32\linewidth}
        \centering
        \includegraphics[width=\linewidth]{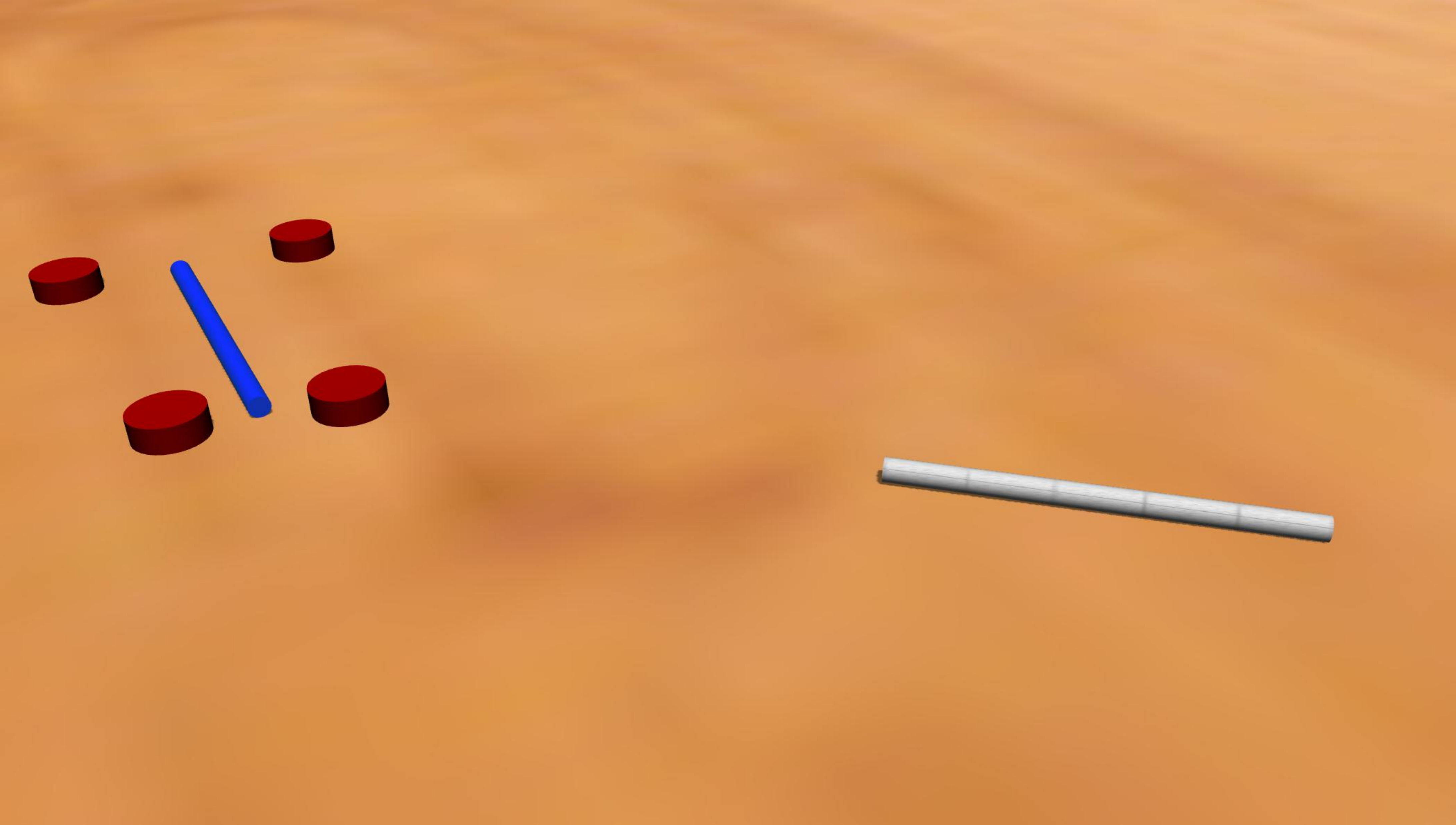}
        \label{fig:start_manip_rod}
    \end{subfigure}
    \hspace{0.35em}
    \begin{subfigure}[t]{.32\linewidth}
        \centering
        \includegraphics[width=\linewidth]{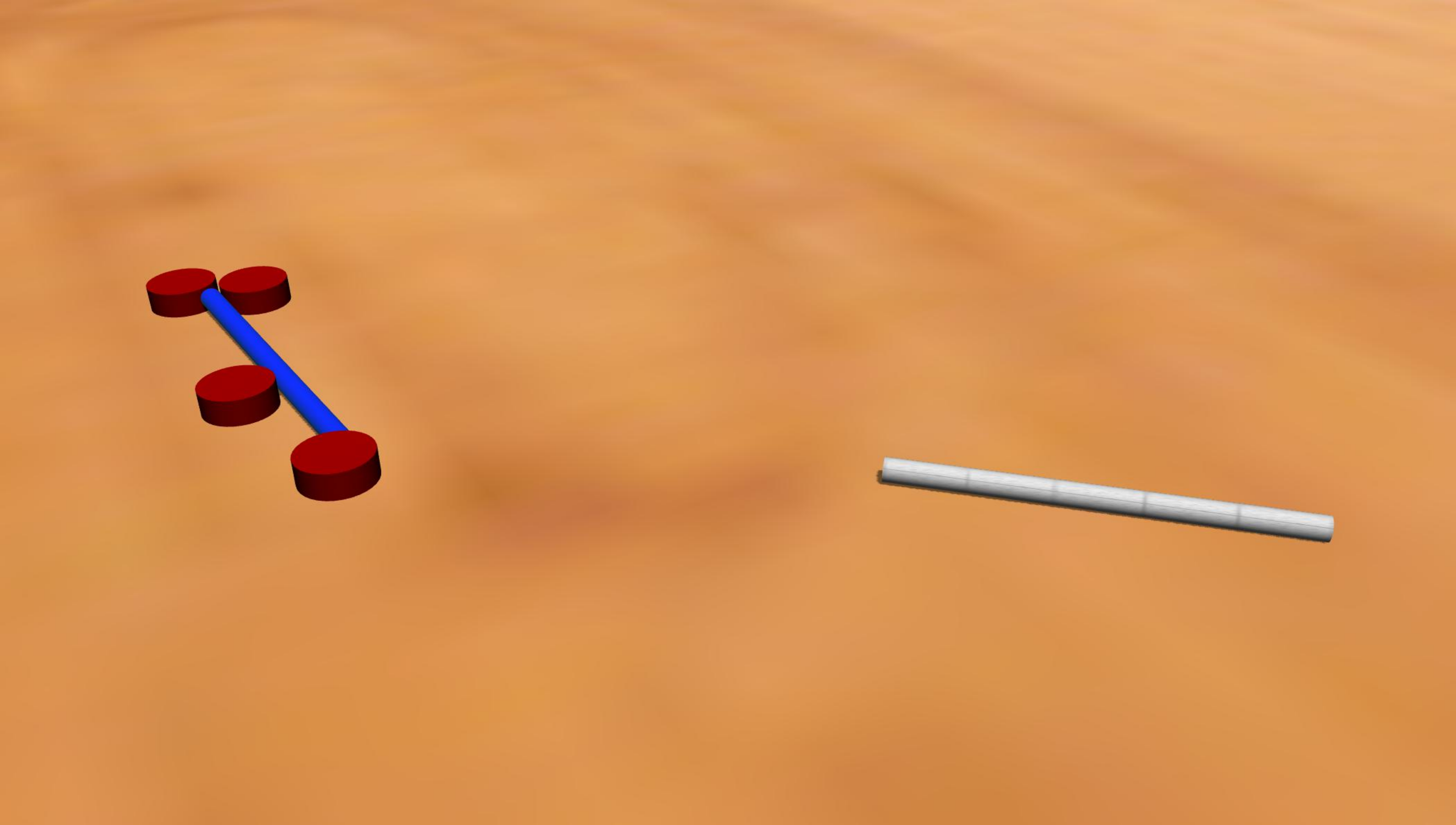}
        \label{fig:contact_manip_rod}
    \end{subfigure}
    \hspace{0.35em}
    \begin{subfigure}[t]{.32\linewidth}
        \centering
        \includegraphics[width=\linewidth]{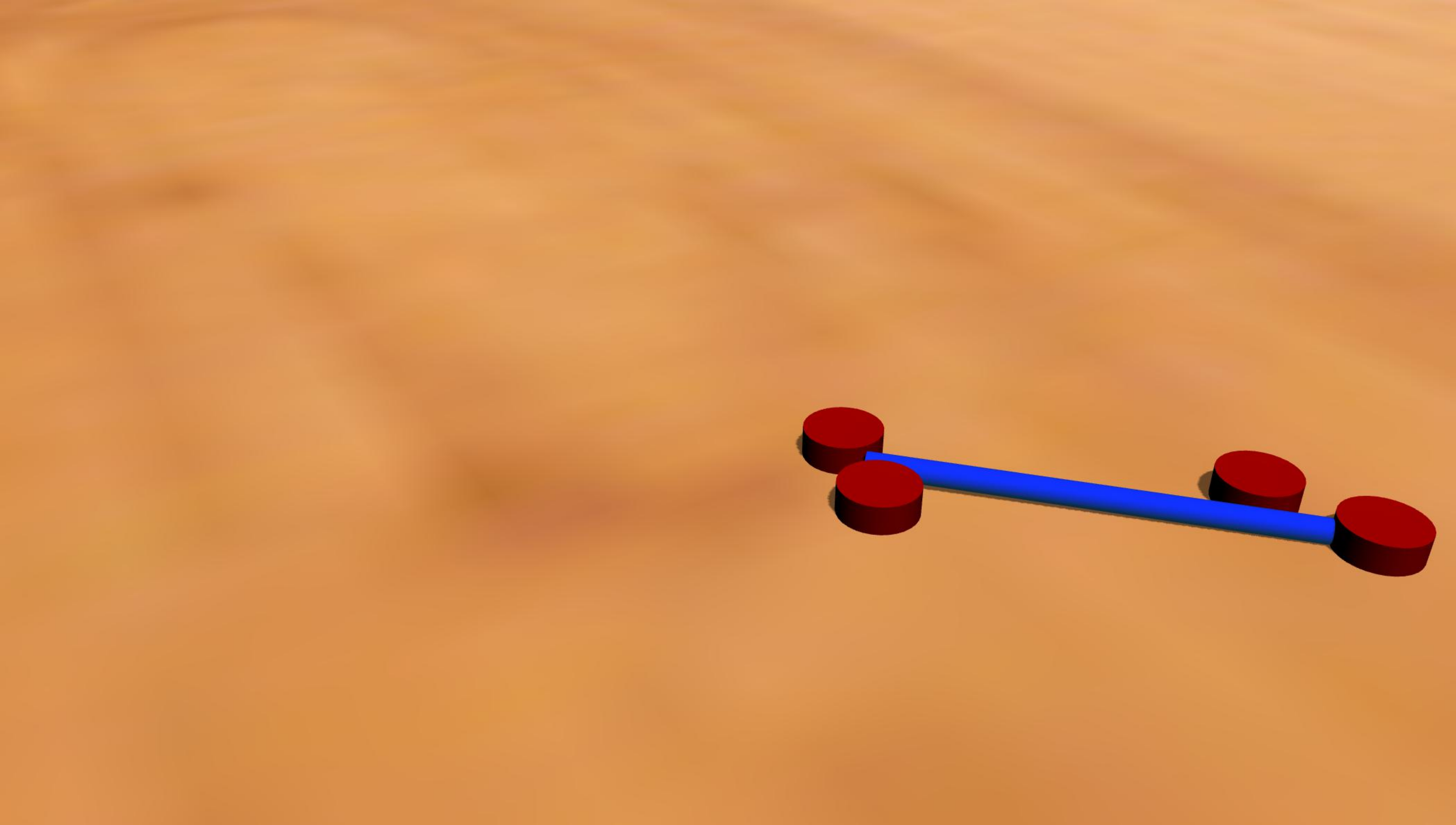}
        \label{fig:end_manip_rod}
    \end{subfigure}
    \caption{A group of $4$ robots (in red) manipulate a rod (in blue) by sliding it along the ground to a desired position and orientation (in gray). The robots begin from arbitrary locations around the rod. (Center) They approach and contact the rod to form a stable grasp, sliding the rod along the ground. (Right) They bring the object to rest at the desired configuration.}
    \label{fig:manip_rod}
\end{figure*}

We consider omnidirectional robots and model the independent (decoupled) dynamics of these robots using a double-integrator dynamics model, given by:
\begin{equation}
    M_{i}\ddot{x}_{i} = \tau_{i},
\end{equation}
where $M_{i}$ denotes the mass matrix of robot $i$, and $\tau_{i}$ denotes its torques. We discretize the dynamics model to obtain the discrete-time dynamics model in \eqref{eq:discrete_dynamics}. 
We constrain the tangential (frictional) components of the contact force using \eqref{eq:contact_tangent}. By representing the contact dynamics between each robot and the object using a set of complementarity constraints, the robots compute their torques from the local optimization problem ($\mcal{P}_{L}$) in \eqref{eq:primal_update},
where: $\bm{v}_{i}(\cdot)$ represents the dynamics constraints of robot $i$ and the rod, $h_{i}(\cdot)$ includes constraints on the initial and desired configuration and velocities of robot $i$ and the object, and $r_{i}(\cdot)$ represents a single set of complementarity constraints describing the contact dynamics of robot $i$ in addition to constraints on its torques and collision-avoidance constraints. 
We consider a quadratic objective function given by
\begin{equation}
\ell_{i}(x_{i},\tau_{i},\bm{f}) = \tau_{i}^{\tp}H_{i}\tau_{i}
\end{equation}
where ${H_{i} \in \mathbb{R}^{n_{\tau,i} \times n_{\tau,i}}}$ denotes a positive definite matrix. The objective function encodes our desire for each robot to manipulate the object using minimal torques, to ultimately minimize energy consumption.

We show the manipulation task in Figure \ref{fig:manip_rod}, where a group of four robots (in red colors) manipulate a blue rod by sliding it along the surface to a desired position and orientation, which is indicated by the gray rod. The robots move the rod through contact before bringing the rod to rest at the end of the manipulation task in Figure \ref{fig:manip_rod}. Videos of the manipulation task are available on the project page.

We examine the success rate of the interior-point solver in solving the contact-implicit optimization problems in the centralized method and DisCo for the manipulation task with $4$ robots. We discretize the optimization problem over $70$ ms intervals, resulting in $1954$ optimization variables in the centralized method, with $1186$ variables per robot in \mbox{DisCo}. A problem is solved successfully if the interior-point solver finds a solution within $5000$ iterations.  We keep the initialization the same across both methods and evaluate the success rate in $100$ manipulation tasks. From Table \ref{table:success_rate_rod_manipulation}, the centralized method attains a $28\%$ success rate, highlighting the difficulty in solving the global contact-implicit trajectory optimization problem ($\mcal{P}_{C}$). In contrast, DisCo achieves a higher success rate of $88\%$, resulting from the decomposition of the global problem into smaller local problems solved by each robot.

\begin{table}[th]
	\centering
	\caption{Success Rate in Solving the $\mathrm{SE}(2)$ Manipulation Problem}
	\begin{tabular}{c c}
		\toprule
		Method & Success Rate ($\%$) \\
		\midrule
		Centralized & $28$ \\ 
		DisCo (Ours) & $88$ \\
		\bottomrule
	\end{tabular}
	\label{table:success_rate_rod_manipulation}
\end{table}

We provide the mean computation time along with the standard deviation using the centralized method and D-COPT in Table \ref{table:computation_time_rod}, with the number of communication rounds in Algorithm~\ref{alg:distributed_update} set at $12$. In DisCo, each robot required a mean time of $0.111$ seconds for each iteration of Algorithm~\ref{alg:distributed_update}, giving a total time of $1.331$ seconds to solve the distributed optimization problem. Meanwhile, the centralized method took a mean computation time of $9.910$ seconds to solve for the torques of all the robots, showing the difficulty in solving the global problem in \eqref{eq:discrete_problem} with all the non-smooth constraints present.

\begin{table}[th]
	\centering
	\caption{Computation Time for the $\mathrm{SE}(2)$ Manipulation Problem}
	\begin{tabular}{c c}
		\toprule
		Method & Time (sec) \\
		\midrule
		Centralized & $7.320 \pm 0.0385$ \\ 
		DisCo (Ours) & $1.331 \pm 0.332$ \\
		\bottomrule
	\end{tabular}
	\label{table:computation_time_rod}
\end{table}

In Figure \ref{fig:normal_contact_force_rod}, we consider a larger-scale task with $16$ robots sliding the rod along a surface and show the normal and tangential components of the contact force applied by each robot. The magnitudes of the normal and tangential components of the contact force indicate the discrete contact interactions between each robot and the object, with non-zero contact forces existing only when the robot makes contact with the object. Further, Figure \ref{fig:normal_contact_force_rod} reveals that only a few robots play a major role in bringing the object to rest at the end of the task, as expected.

\begin{figure}[th]
	\centering
	\includegraphics[width=0.8\columnwidth]{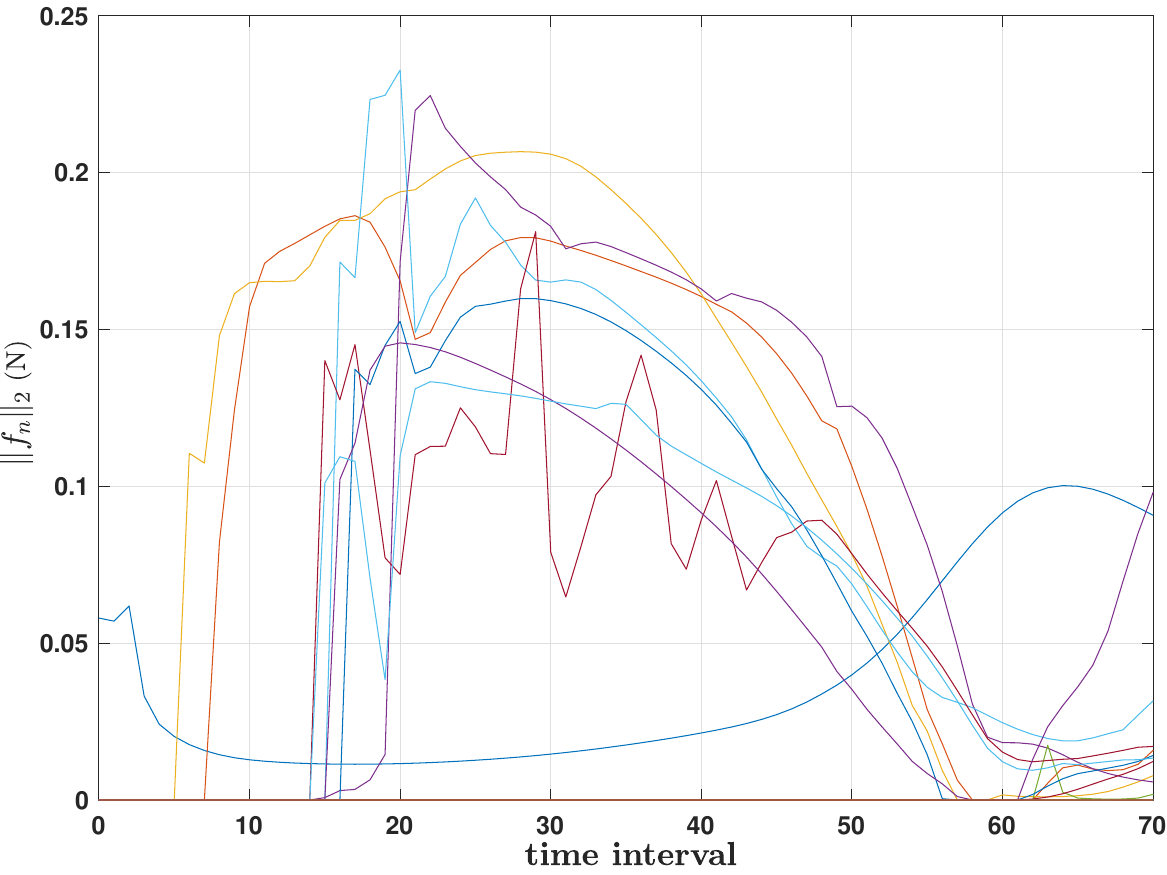}
	\includegraphics[width=0.8\columnwidth]{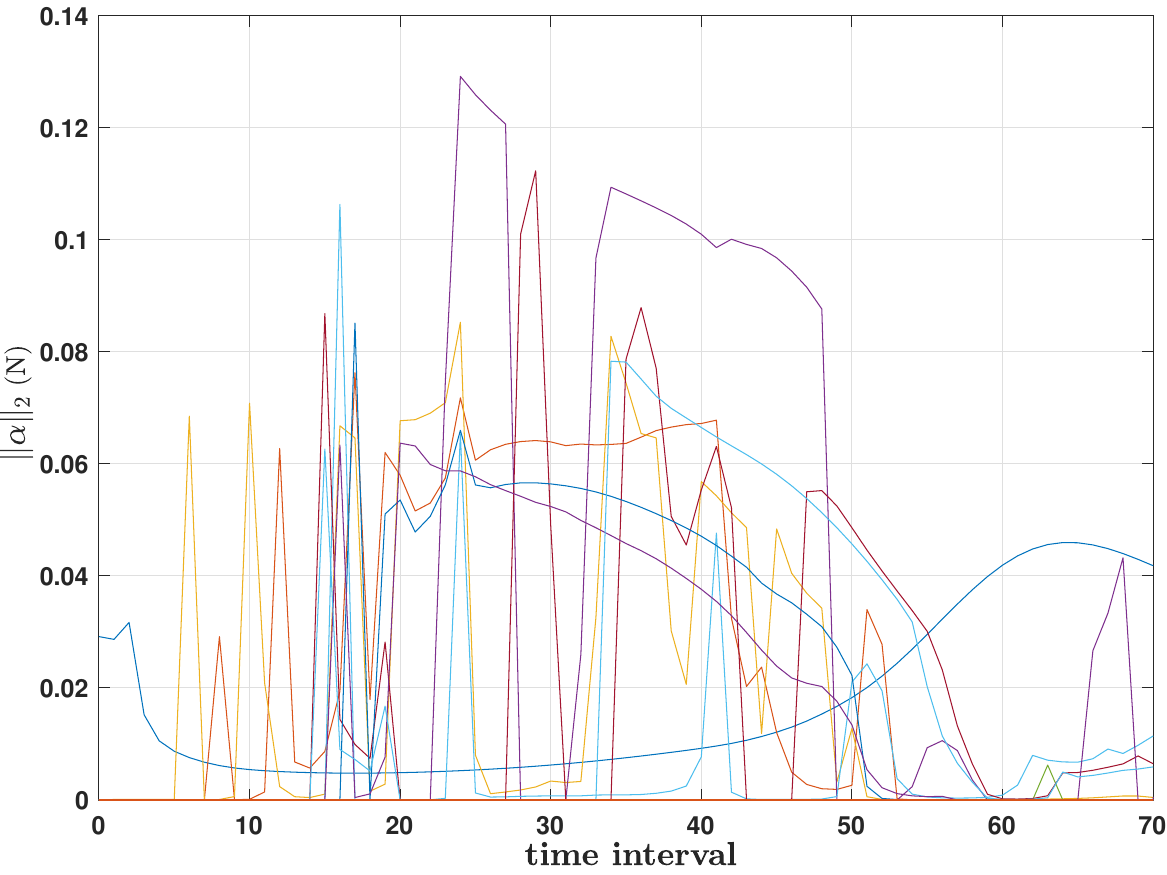}
	\caption{The normal components (top) and tangential components (bottom) of the contact forces applied by each robot, as $16$ robots slide a rod to a desired position and orientation along a surface, showing the discrete contact interactions between each robot and the object as the object slides along its surface.}
	\label{fig:normal_contact_force_rod}
\end{figure}

\subsubsection{$\mathrm{SE}(3)$ Manipulation} %
\label{sec:object_manipulation}
In this manipulation task, $N$ robots manipulate an object falling under gravity, through contact, after it has been dropped from an elevated platform, bringing the object to rest at a desired configuration, specified by a desired transformation ${\mathbb{T}_{\mathrm{des}} \in \SE(3)}$. We represent the orientation of the object at time $t$ by a rotation matrix ${R(t) \in \mathrm{SO}(3)}$ and its angular velocity by ${\omega \in \mathbb{R}^{3}}$. By applying an angular velocity $\omega$ on the object through contact over a time interval of length $\Delta t$, the object rotates through an angle ${\phi \in \mathbb{R}}$ around the rotation axis ${u \in \mathbb{R}^{3}}$. With the assumption that the angular velocity remains constant over each time interval, the change in the rotation  matrix representing the orientation of the object is described by Rodriques' formula,
\begin{equation}
R_{dt} = \cos(\phi)I + (1 - \cos(\phi))u u^{\tp} + \sin(\phi) \hat{\omega}
\end{equation}
where ${I \in \mathbb{R}^{3 \times 3}}$ represents the identity matrix and ${\hat{\omega} \in \mathbb{R}^{3 \times 3}}$ represents a skew-symmetric matrix derived from $\omega$
with ${\phi = \norm{\omega \cdot \Delta t}_{2}}$ and ${u = \frac{\omega \cdot \Delta t}{\phi}}$. The resulting dynamics for the rotation matrix is given by
\begin{equation}
R(t + 1) = R_{dt}R(t),
\end{equation}
with ${R(t + 1) \in \mathrm{SO}(3)}$, which remains a valid rotation matrix.

\begin{figure*}[t]
    \centering
    \begin{subfigure}[t]{.32\linewidth}
        \centering
        \includegraphics[width=\linewidth]{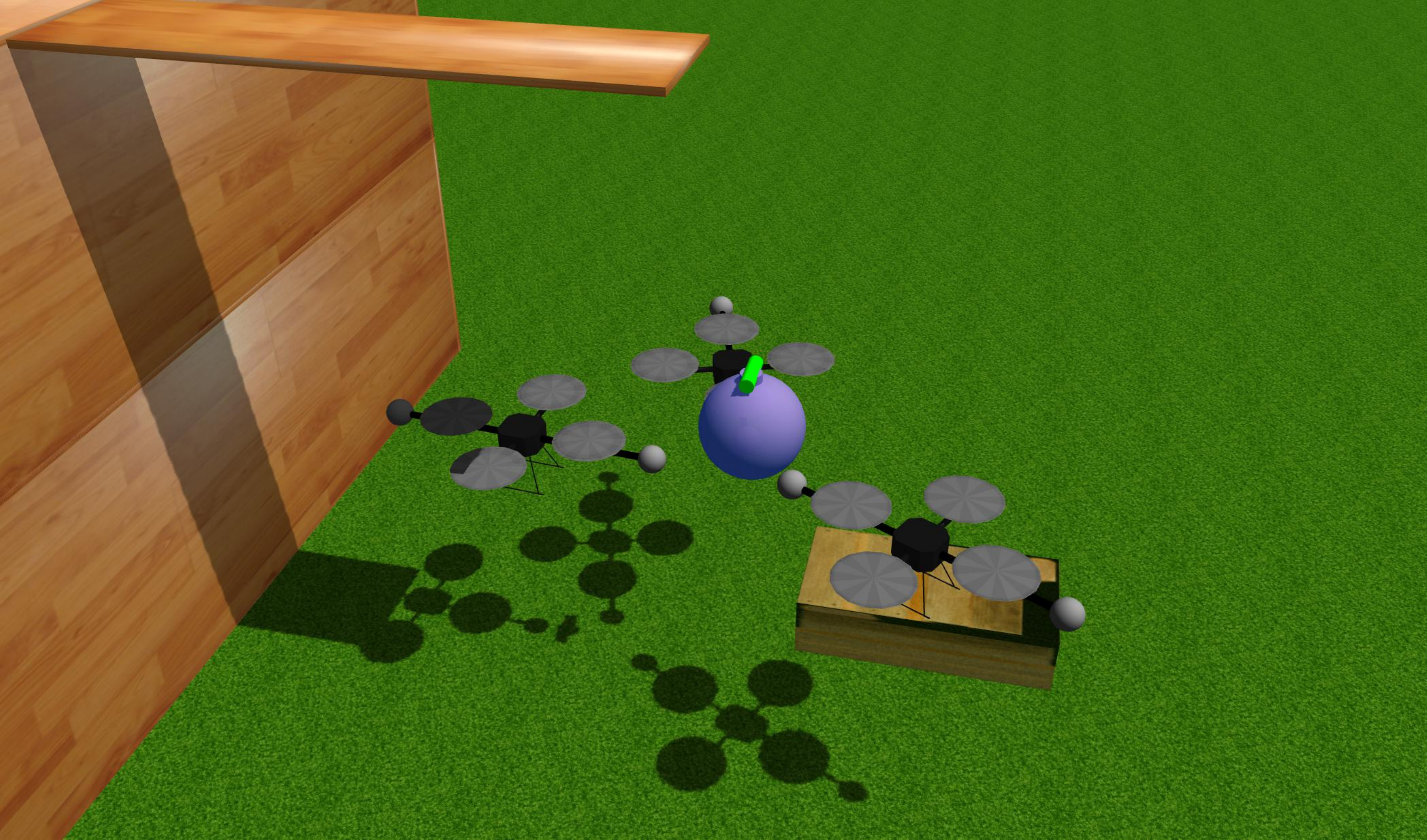}
        \label{fig:start_manip_object}
    \end{subfigure}
    \hspace{0.35em}
    \begin{subfigure}[t]{.32\linewidth}
        \centering
        \includegraphics[width=\linewidth]{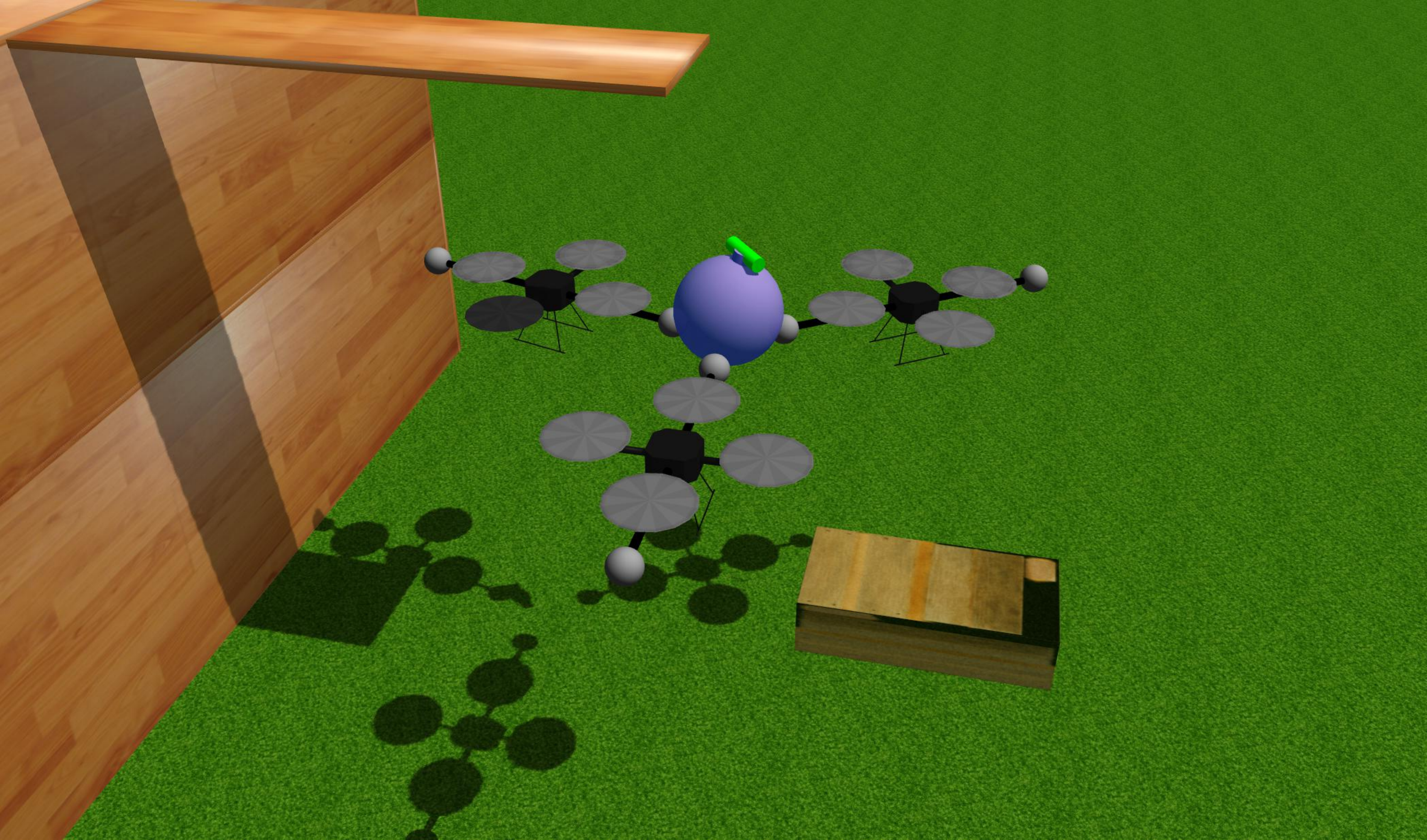}
        \label{fig:manip_object}
    \end{subfigure}
    \hspace{0.35em}
    \begin{subfigure}[t]{.32\linewidth}
        \centering
        \includegraphics[width=\linewidth]{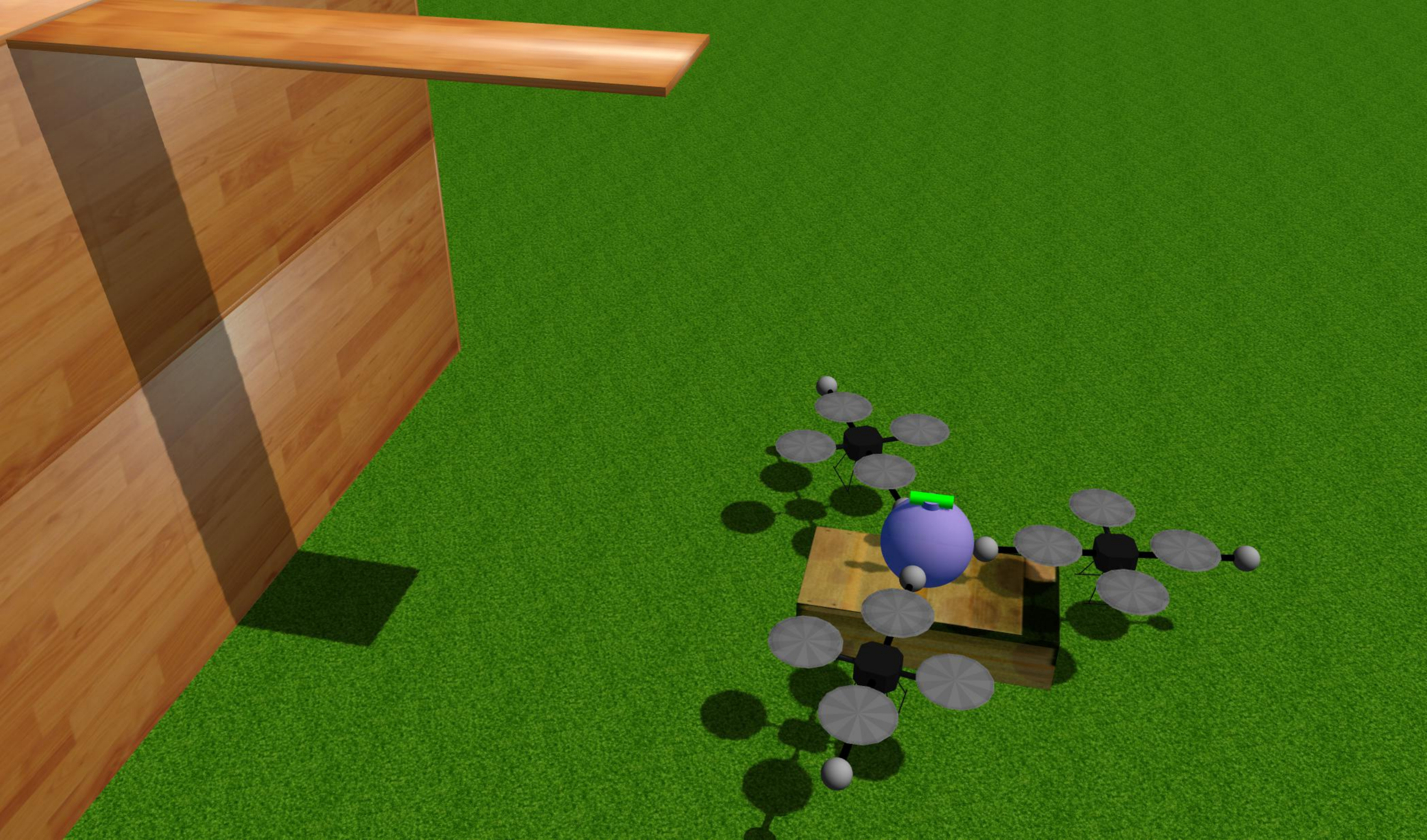}
        \label{fig:end_manip_object}
    \end{subfigure}
    \caption{A group of $3$ quadrotors manipulate a spherical object (purple) released from a platform, as it falls under gravity, to a desired position and orientation on the pedestal. (Left) The quadrotors approach the falling object, (center) gently make contact, and decelerate its fall, and (right) bring the object to rest at its desired configuration on the pedestal.}
    \label{fig:manip_object}
\end{figure*}

To compute the torques required to manipulate the object, each robot solves the local optimization problem in \eqref{eq:primal_update} where the configuration of the object ${x_{\obj}(t) \in \mathrm{SE}(3)}$ is parameterized by its position ${x_{\obj,\pos} \in \mathbb{R}^{3}}$ and orientation, represented with a rotation matrix ${x_{\obj,\orn} \in \mathrm{SO}(3)}$, relative to a fixed inertial frame, and ${\dot{x}_{\obj}(t) \in \mathbb{R}^{6}}$ includes its linear and angular velocities at time $t$ with a single set of complementarity constraints for its contact dynamics. We specify the initial and desired configuration and velocities of robot $i$ and the object in $h_{i}(\cdot)$.

In Figure \ref{fig:manip_object}, a group of $3$ quadrotors manipulate an object as it falls under gravity after it has been released from an elevated platform. The quadrotors contact the object to support and manipulate it to its desired position and orientation, which can be observed from the orientation of the green cap on the object's top. The quadrotors bring the object to rest on the wooden box in Figure \ref{fig:manip_object}. 
We discretize the problem over $70$ ms intervals, resulting in $2430$ optimization variables in the centralized method, with $2082$ variables per robot in D-COPT. We provide the mean computation time, and its standard deviation, of DisCo and the centralized method in Table \ref{table:computation_time_object}. In DisCo, each robot took a mean time of $17.510$ seconds to compute its torques, with each of the $12$ iterations in Algorithm~\ref{alg:distributed_update} requiring an average of $1.459$ seconds. This optimization problem is significantly larger than the problem in Section \ref{sec:rod_manipulation}, therefore requiring more computation time. In contrast, the centralized method required $49.782$ seconds to solve for the torques of all robots, and repeatedly fails to find a solution in many instances of the problem, as noted in Section \ref{sec:rod_manipulation}.

\begin{table}[th]
	\centering
	\caption{Computation Time for the $\mathrm{SE}(3)$ Manipulation Problem}
	\begin{tabular}{c c}
		\toprule
		Method & Time (sec) \\
		\midrule
		Centralized & $49.782 \pm 0.559$ \\ 
		DisCo (Ours) & $17.510 \pm 10.312$ \\
		\bottomrule
	\end{tabular}
	\label{table:computation_time_object}
\end{table}

In Figure \ref{fig:normal_contact_force_object}, we consider a larger-scale problem with $12$ robots manipulating an object released from an elevated platform. The normal and tangential components of the contact forces highlight the discrete contact interactions between each robot and the object. In contrast to the $\mathrm{SE}(2)$ manipulation problem, all the robots involved in the task share relatively the same level of responsibility in bringing the object to rest on the wooden box, as expected in this task.

\begin{figure}[!ht]
	\centering
	\includegraphics[width=0.8\columnwidth]{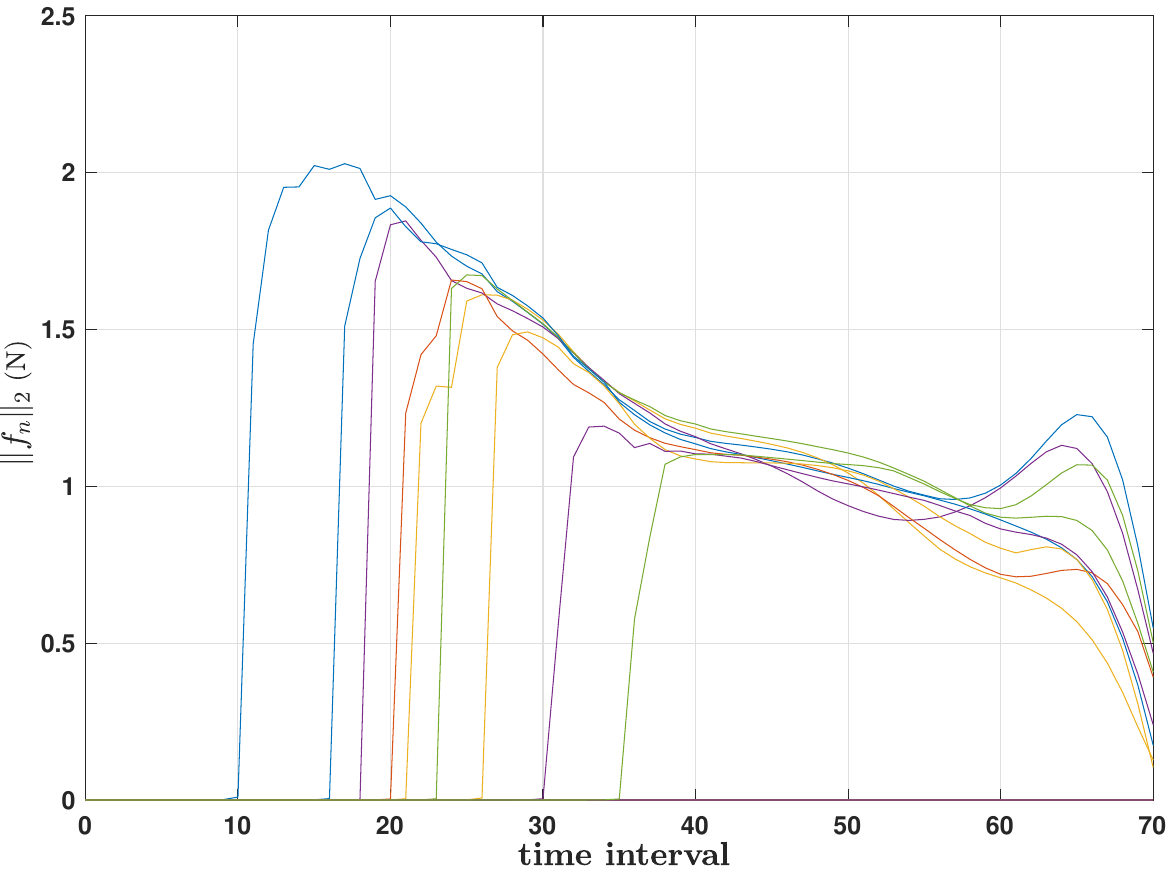}
	\includegraphics[width=0.8\columnwidth]{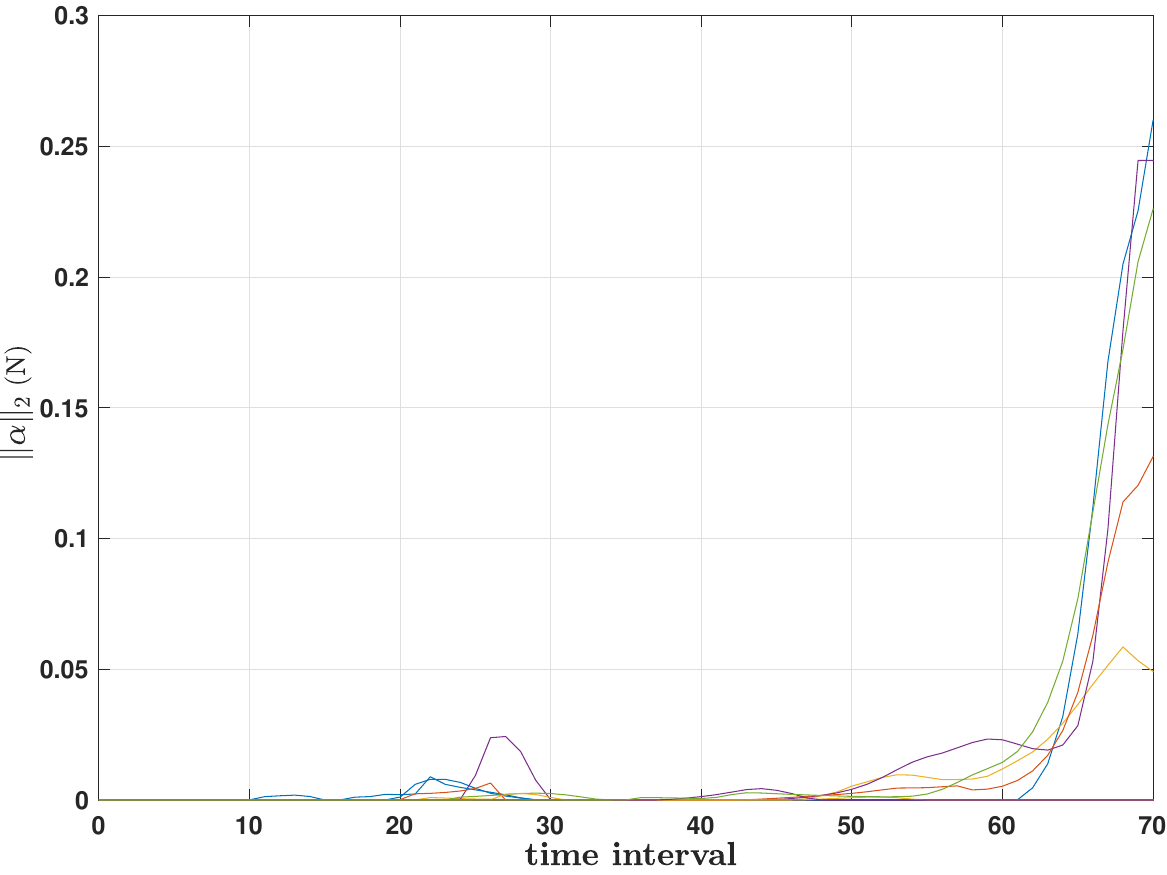}
	\caption{The normal components (top) and tangential components (bottom) of the contact forces applied by each quadrotor on the object as $12$ quadrotors manipulate an object released from an elevated platform, as it falls under gravity. Each quadrotor makes discrete contact interactions with the object to manipulate the object.}
	\label{fig:normal_contact_force_object}
\end{figure}

\begin{figure*}[th]
    \centering
    \begin{subfigure}[t]{.238\linewidth}
        \centering
        \includegraphics[width=\linewidth]{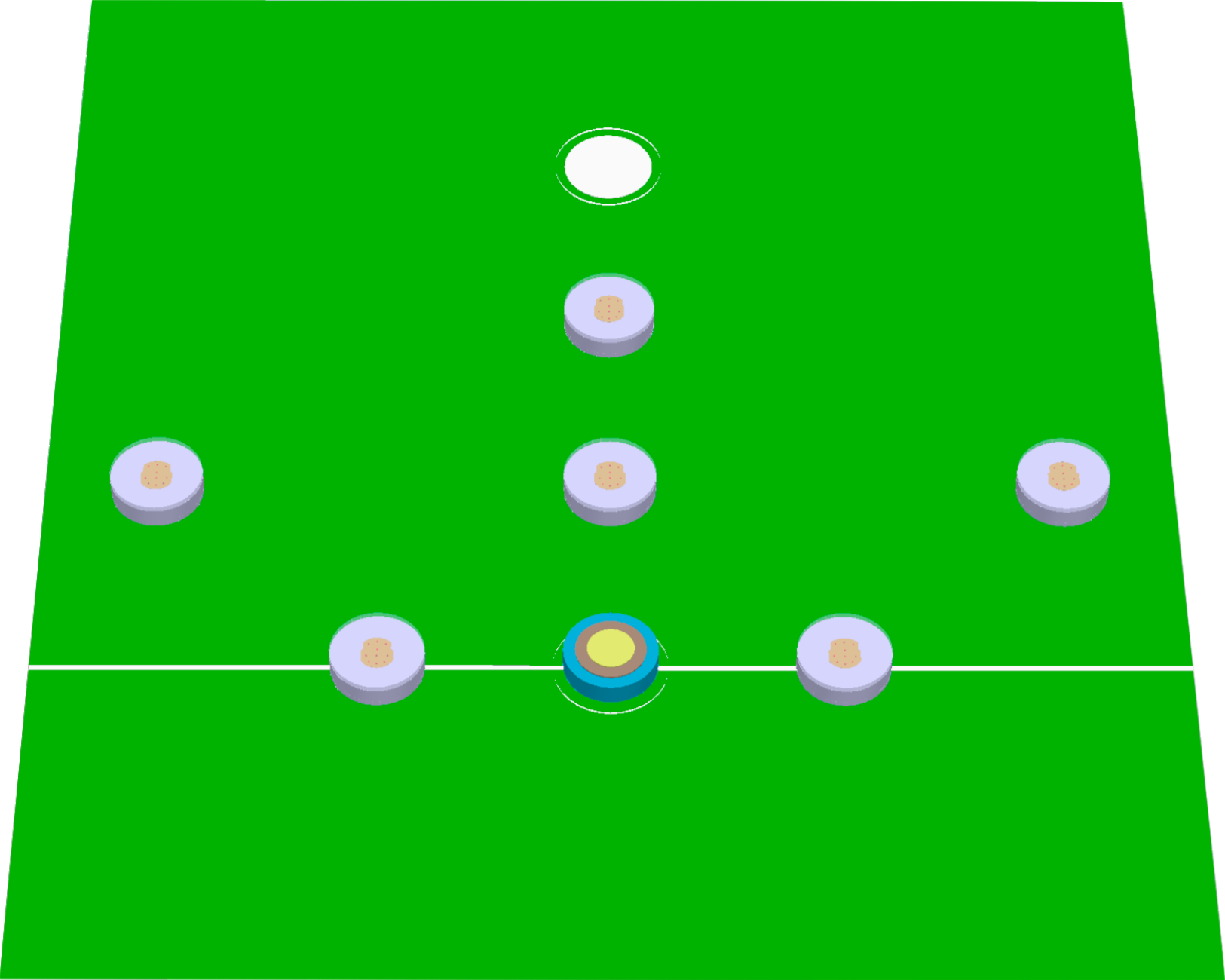}
        \label{fig:start_robot_game}
    \end{subfigure}
    \hspace{0.2em}
    \begin{subfigure}[t]{.238\linewidth}
        \centering
        \includegraphics[width=\linewidth]{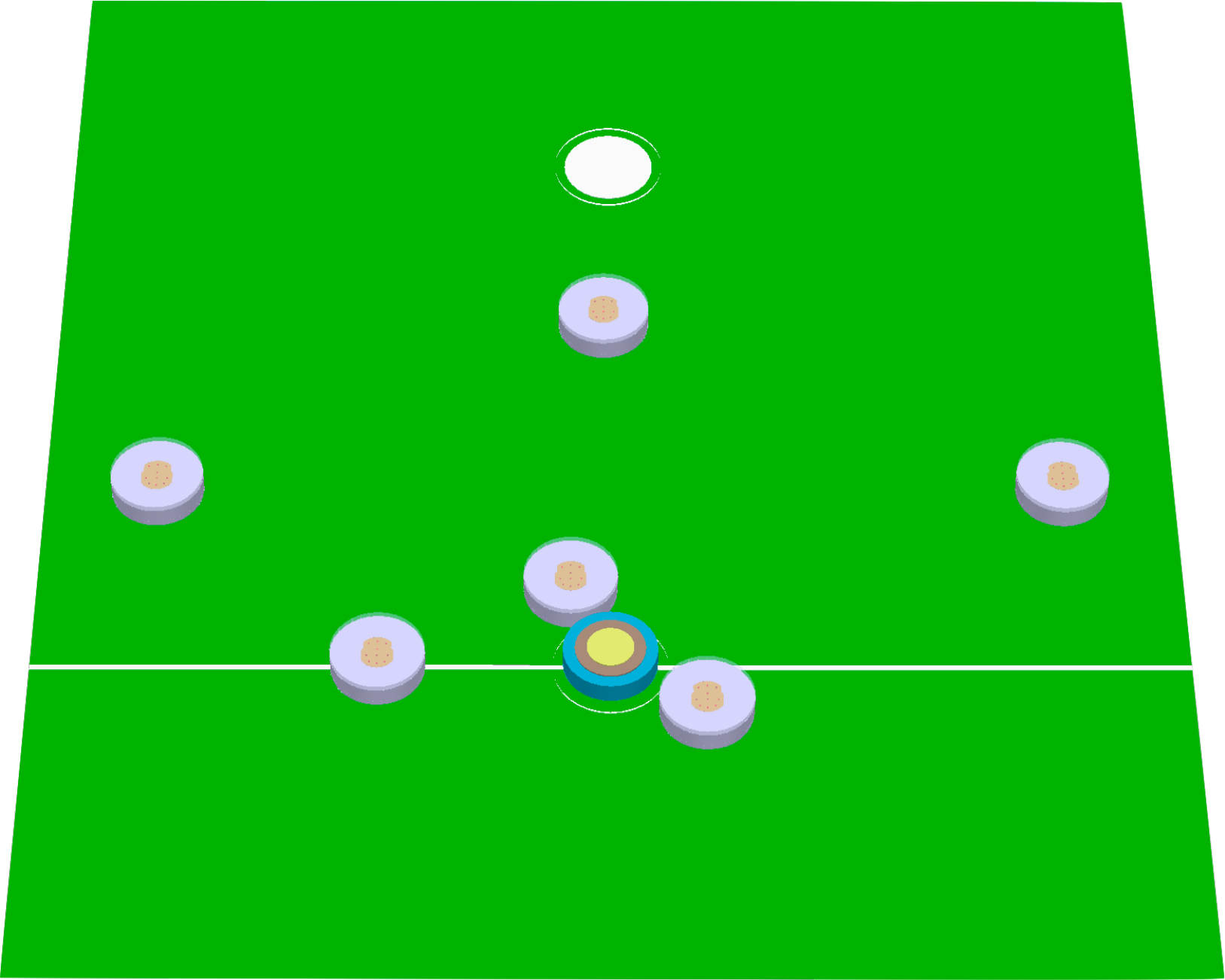}
        \label{fig:short_robot_game}
    \end{subfigure}
    \hspace{0.2em}
    \begin{subfigure}[t]{.238\linewidth}
        \centering
        \includegraphics[width=\linewidth]{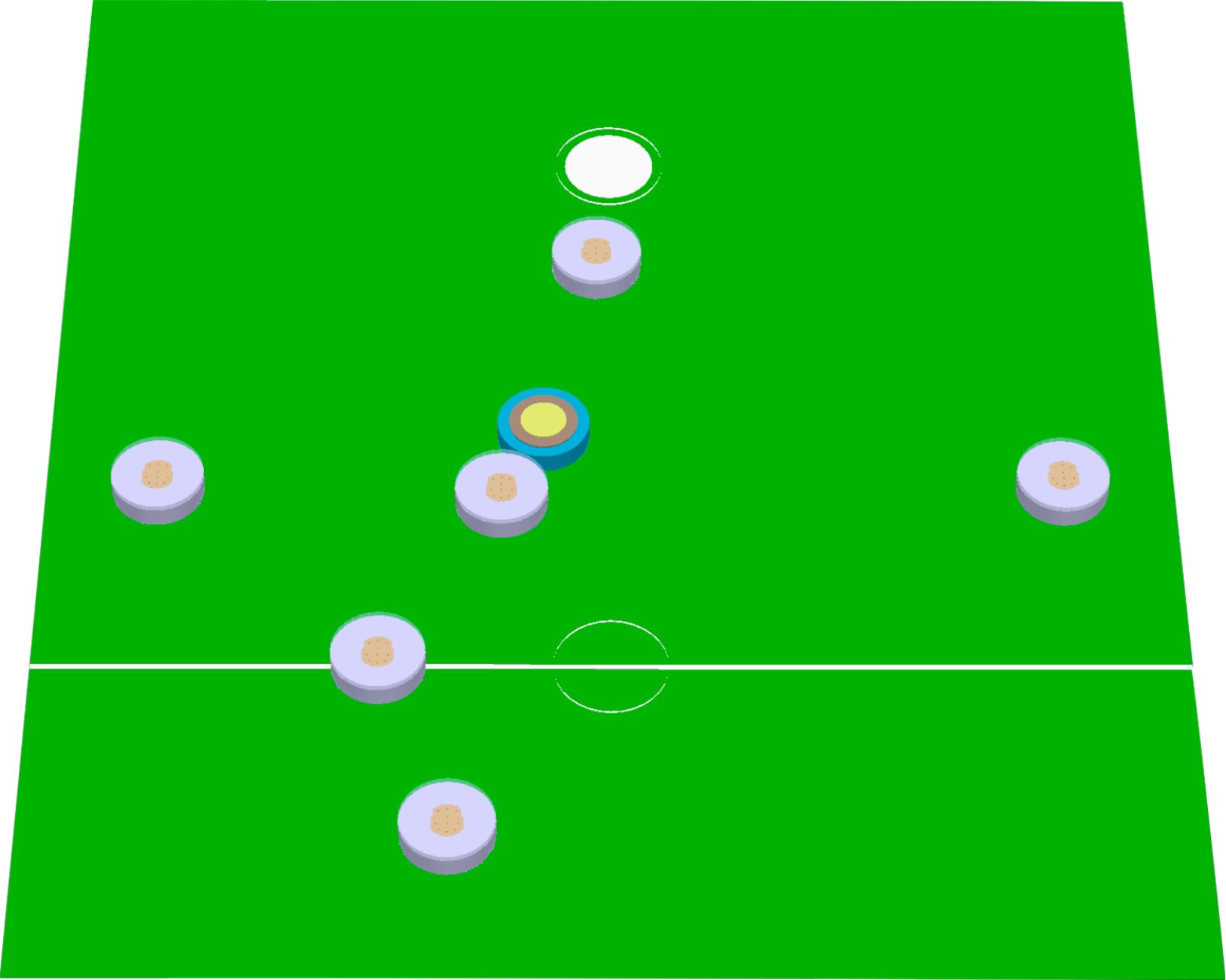}
        \label{fig:long_robot_game}
    \end{subfigure}
    \hspace{0.2em}
    \begin{subfigure}[t]{.238\linewidth}
        \centering
        \includegraphics[width=\linewidth]{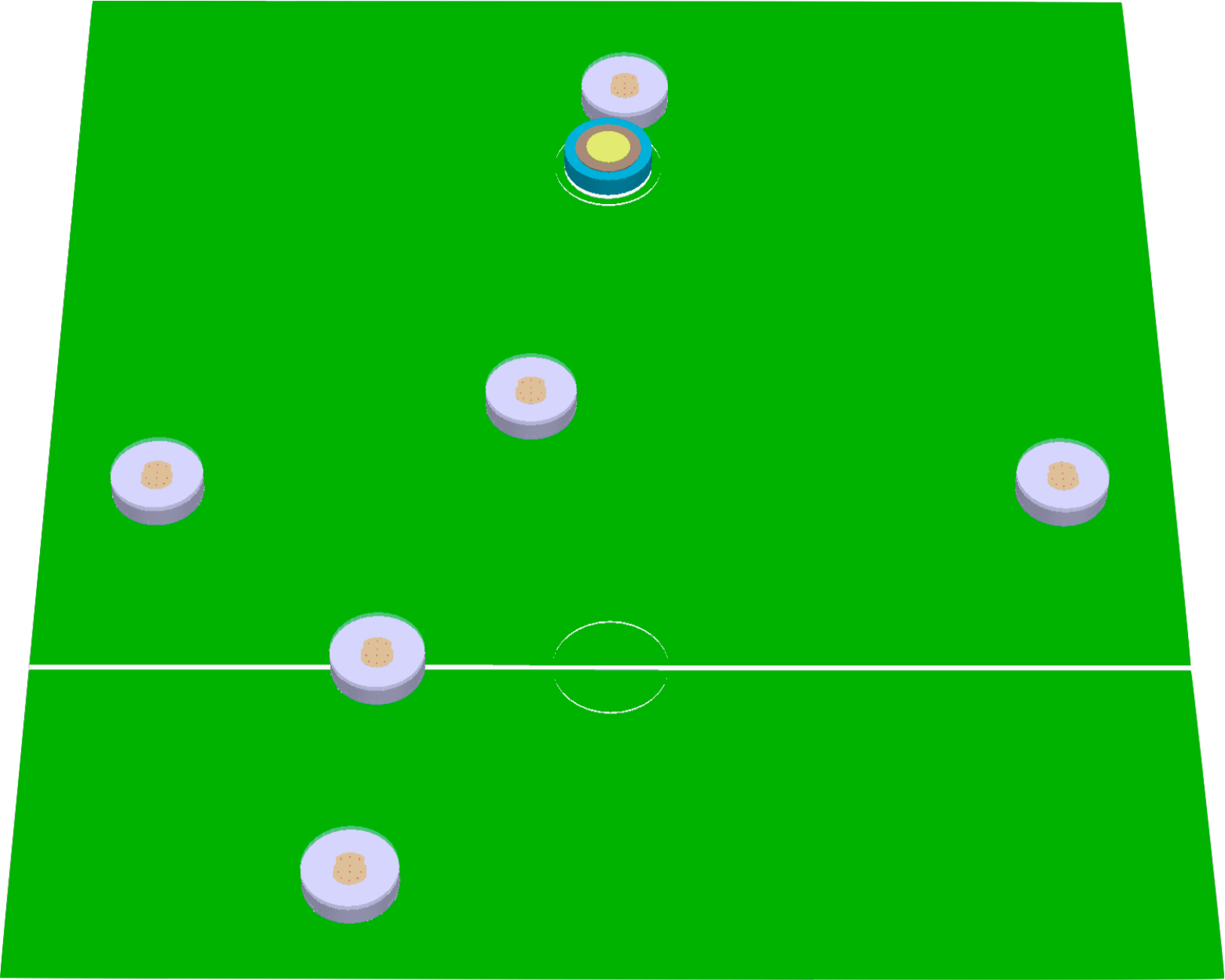}
        \label{fig:goal_robot_game}
    \end{subfigure}
    \caption{A team of robots (in purple) in a hit-and-catch robot team game, with the objective of getting the blue puck (with a yellow top) into the goal (designated by the white circle). 
    (Left) A team of purple robots lineup in a hit-and-catch game before kickoff. (Center-left) A pair of robots complete a short pass to begin the game, after kickoff. (Center-right) A robot attempts a long pass to a teammate, who is close to the goal area. (Right) The teammate traps the puck within the goal area, scoring a point.}
    \label{fig:robot_game_sim}
\end{figure*}

\subsection{Tactics Plannings in Robot Team Sports}
We consider distributed tactics planning in a team of $N$ robots, where each robot passes an object (\emph{puck}) to a teammate by applying an impulse to the object, while the teammate traps the puck by applying an impulse that arrests its motion.
In general, in cooperative hit-and-catch robot team games, the robots move an object into the net in a time-optimal manner, which can be encoded in the objective or constraint function.
We consider planar hit-and-catch games with omnidirectional robots in a 2D workspace and a puck constrained to the plane, but the methods we escribe can be directly extended to games in which the puck or ball moves in a 3D workspace (e.g., soccer). We denote the configuration of the puck as ${x_{\obj} \in \mathbb{R}^{2}}$, representing its position, and in addition, denote the velocity of the puck as ${\dot{x}_{\obj} \in \mathbb{R}^{2}}$. Likewise, we denote the configuration of robot $i$ as ${x_{i} \in \mathbb{R}^{2}}$.
As in Section \ref{sec:rod_manipulation}, we model the independent (decoupled) dynamics of these robots using a double-integrator dynamics model, given by:
\begin{equation}
    M_{i}\ddot{x}_{i} = \tau_{i},
\end{equation}
where $M_{i}$ denotes the mass matrix of robot $i$, and $\tau_{i}$ denotes its torques. Likewise, we discretize the dynamics model to obtain the discrete-time dynamics model in formulating the contact-implicit trajectory optimization problem in \eqref{eq:discrete_problem}. 

We can represent the contact dynamics between the robots and the puck using the set of constraints in \eqref{eq:contact_normal}, \eqref{eq:contact_complementarity}, and \eqref{eq:contact_tangent}, and subsequently, formulate a contact-implicit trajectory optimization problem for this task, taking the same form in \eqref{eq:discrete_problem}. To compute the torques required to complete the task, each robot solves its local optimization problem, given by \eqref{eq:primal_update}, consisting of a single set of complementarity constraints describing its contact dynamics as well as constraints on its initial and desired configurations and that of the puck, while communicating with its neighbors. In this task, we utilize an energy-minimization objective function, similar to that specified in Section \ref{sec:rod_manipulation}.

Figure \ref{fig:robot_game_sim} shows a hit-and-catch robot team game with six robots. The objective of the game is to get the blue puck with the yellow top into the goal area, designated by the white circle, which results in the robots scoring a point. After kickoff, a pair of robots begin the game with a short pass in Figure \ref{fig:robot_game_sim} (Left). The robots pass the puck among themselves to collaboratively get the puck into the goal area. In Figure \ref{fig:robot_game_sim} (Center), a robot makes a long pass to a teammate, who is close to the goal area. The teammate traps the puck within the goal area, completing the pass and scoring a point for the team, depicted in Figure~\ref{fig:robot_game_sim} (Right).

Figure \ref{fig:normal_contact_force_robot_game} shows the normal components of the contact force applied by each robot to the puck during the game. The short-duration spikes in the normal contact force in Figure \ref{fig:normal_contact_force_robot_game} results when a robot hits the puck to attempt a pass, or traps the puck to complete a pass.

\begin{figure}[!ht]
    \centering
    \includegraphics[width=0.8\columnwidth]{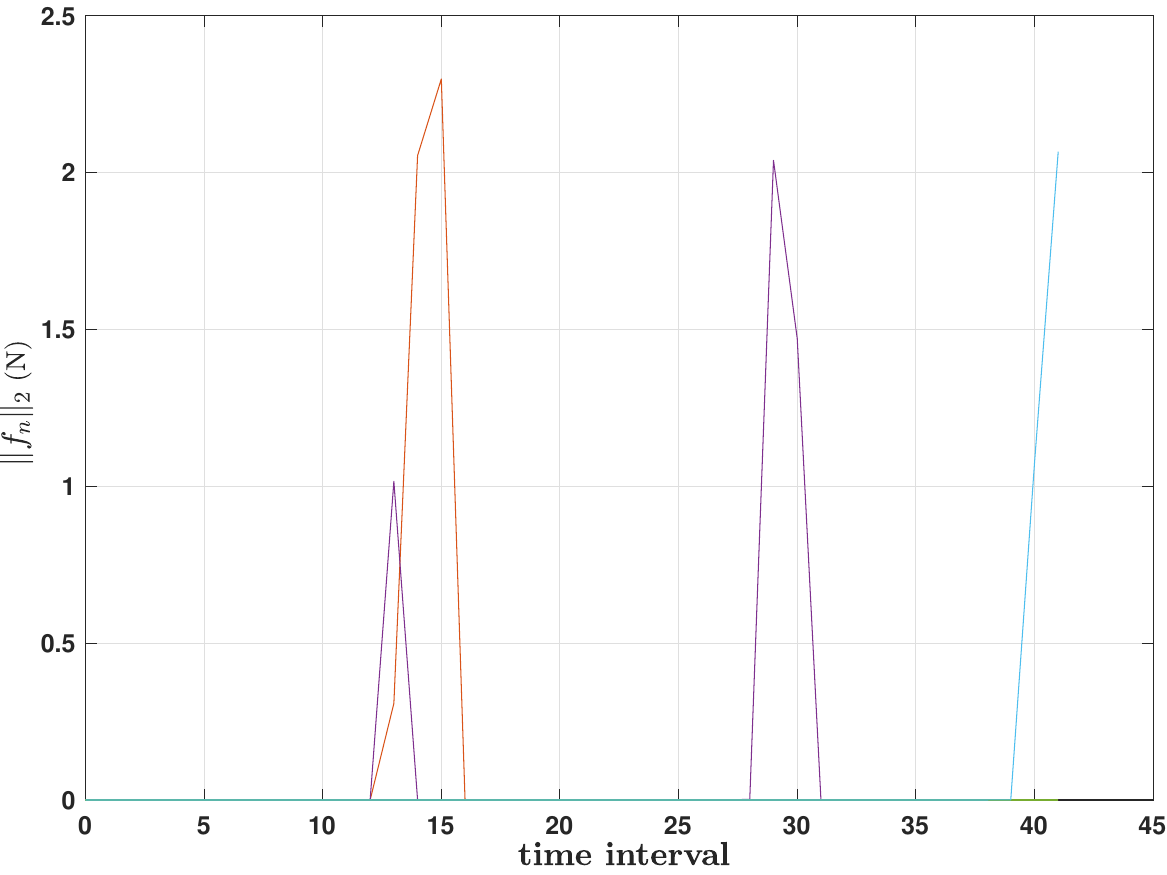}
    \caption{The normal components of the contact forces applied by each robot on the puck in the hit-and-catch robot team game, shown in Figure \ref{fig:robot_game_sim}. A spike in the normal contact force occurs when a robot hits or traps the puck, when attempting or completing a pass.}
    \label{fig:normal_contact_force_robot_game}
\end{figure}

        \section{Hardware Experiments}
\label{sec:experiments}

We evaluate DisCo in hardware experiments. Here, we consider a locomotion problem where we seek to compute the necessary torques to move a modular robot over a terrain, leveraging contact with its environment to accomplish the task. We assume the modular-robot consists of $N$ roller modules, where each module represents an individual agent, with each agent computing its torques locally, while collaborating with other modules. Within the framework of DisCo, we assume that each roller module represents an individual agent with onboard computation resources. 

\smallskip
\noindent \textbf{Truss Robot.}
We utilize a truss-like soft robot that changes shape by the nodes of a constant-perimeter structure \cite{usevitch_untethered_2020}; we refer to it here as the truss robot. The truss robot consists of a $3.4$-meter long tube, inflated with air to a pressure of about $40~\mathrm{kPa}$. The tube is made up of an outer fabric layer enclosing an air-tight bladder. In our experiments, we consider a truss robot with three nodes in the shape of a triangle, depicted in the Figure~\ref{fig:truss_robot_three}. Each node consists of a roller module, two of each are actuated by an electric motor, which power the shape change of the robot.

\begin{figure}[!ht]
    \centering
    \includegraphics[width=0.7\columnwidth, trim={0, 3.7cm, 0, 1.5cm}, clip]{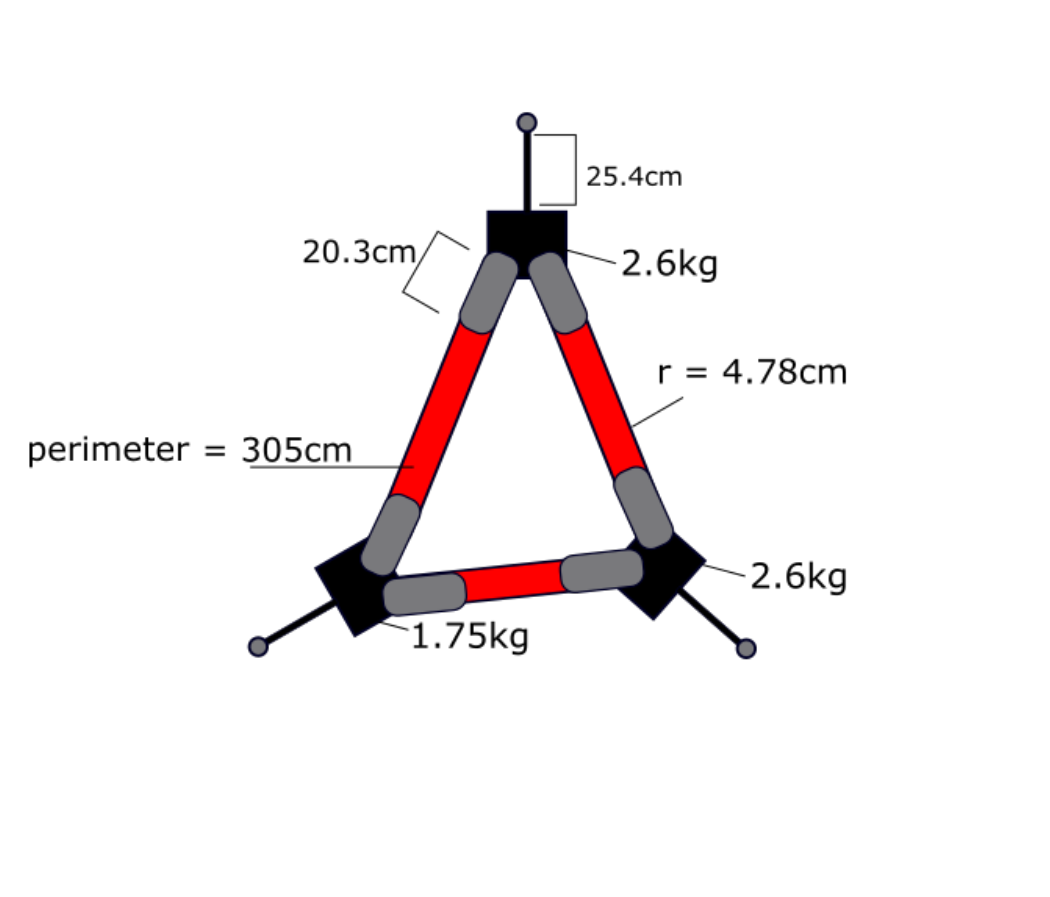}
    \caption{A Truss robot, consisting of three roller modules at each vertex.}
    \label{fig:truss_robot_three}
\end{figure}

\smallskip
\noindent \textbf{Problem Formulation.} 
Given this setup, we can formulate the contact dynamics using the set of constraints in \eqref{eq:contact_normal}, \eqref{eq:contact_complementarity}, and \eqref{eq:contact_tangent}, where the contact surface represents the ground terrain. The contact-implicit trajectory optimization problem takes the same as form as \eqref{eq:discrete_problem}, where the objective function is separable among the $N$ modules. Typically, the objective function consists of components relating to energy consumption, in addition to other components. Further, we note that for robotic locomotion problems, the constraint functions $\bm{h}$ and $\bm{r}$ include geometric constraints, specifying the geometry of the modular robot, including the connection between different modules.

We represent the dynamics of the truss robot using the model:
\begin{equation}
    m_{i} \ddot{x}_{i} = \sum_{j \in \mcal{N}_{i}} \tau_{i, j} \frac{x_{i} - x_{j}}{\norm{x_{i} - x_{j}}} + f_{c, i} + g_{i},
\end{equation}	
where ${m_{i} \in \mathbb{R}}$ denotes the mass of the module, ${\tau_{i, j} \in \mathbb{R}}$ denotes the torque applied at the prismatic joint along edge ${e = (i, j)}$ connecting nodes $i$ and $j$, ${f_{c, i}}$ denotes the contact force applied to node $i$, and $g_{i}$ denotes its gravitational force. In this model, we assume that each edge represents a prismatic link, i.e., the truss robot changes the lengths of its edges by applying a torque in the direction of the associated edge.
We execute DisCo to compute reference trajectories for the nodes of the truss robot for locomotion on a flat terrain from an initial location to a desired configuration. We express the desired configuration of the truss robot in terms of the location of its center-of-mass.

\smallskip
\noindent \textbf{Hardware Implementation.}
For locomotion on flat terrain, we selected a matte fiber-reinforced plastic sheet (Home Depot No.~63003) to create suitable visual contrast. Further, we reinforced the truss robot with aluminum angle extrusions in order to resist potential out-of-plane motion. Because only two of the three nodes contain motors, we added an additional $1~\mathrm{kg}$ of mass to the motor-free node to equalize all three node weights. 
We communicate with the truss robot over an nRF24l01+ radio attached to a Microsoft Surface Book 2 laptop. Each of the truss robot nodes listens for commands from the laptop to rotate in order to lengthen or shorten edge lengths. The front camera of the laptop is used with Multitracker and OpenCV \cite{opencv_library} computer vision libraries. 
We send commands at a frequency of $2~\mathrm{Hz}$ and compare the measured length between each of the three nodes of the truss robot and that of the optimized trajectory to determine the control commands.
The truss robot motors are set to move at a constant speed until the node is measured within $10\%$ of the optimized trajectory. Once two of the edges achieve their target length, the process loops to the next set of edge lengths as given by the optimized trajectory. All of these measured and targeted edge lengths are recorded in a time series through the script, and processed in MATLAB. 

We provide qualitative results in Figure~\ref{fig:truss_robot}, where we show the truss robot in a rolling gait as it moves along a flat terrain from its initial configuration (with its center-of-mass position denoted in green) to a desired configuration (with its center-of-mass denoted in yellow). In the top row, we show the optimized trajectory of the truss robot, computed in simulation. We assign a unique color to each node and display each node's estimate of the composite shape of the truss robot. From Figure~\ref{fig:truss_robot}, we note that all nodes achieve consensus on the optimized trajectory of the truss robot, with the triangles (representing the truss) computed by each robot overlapping in each frame of the video. Further, in the bottom row in Figure~\ref{fig:truss_robot}, we provide time-matched frames of the robot executing the reference trajectory in the real-world. The frames from the real-world experiment closely mirror the simulation results.
In addition, we show the tracking results in Figure~\ref{fig:truss_robot_quant}. The truss robot exhibits good tracking performance while executing the rolling gait. The observed tracking errors may be due to drift from stretching of the soft inflated body or out-of-plane effects not captured by the single camera system used in computing the control commands. Future work will seek to address these challenges by implementing a higher-fidelity low-level tracking controller.

\begin{figure}[ht]
	\centering
	\includegraphics[width=0.95\columnwidth]{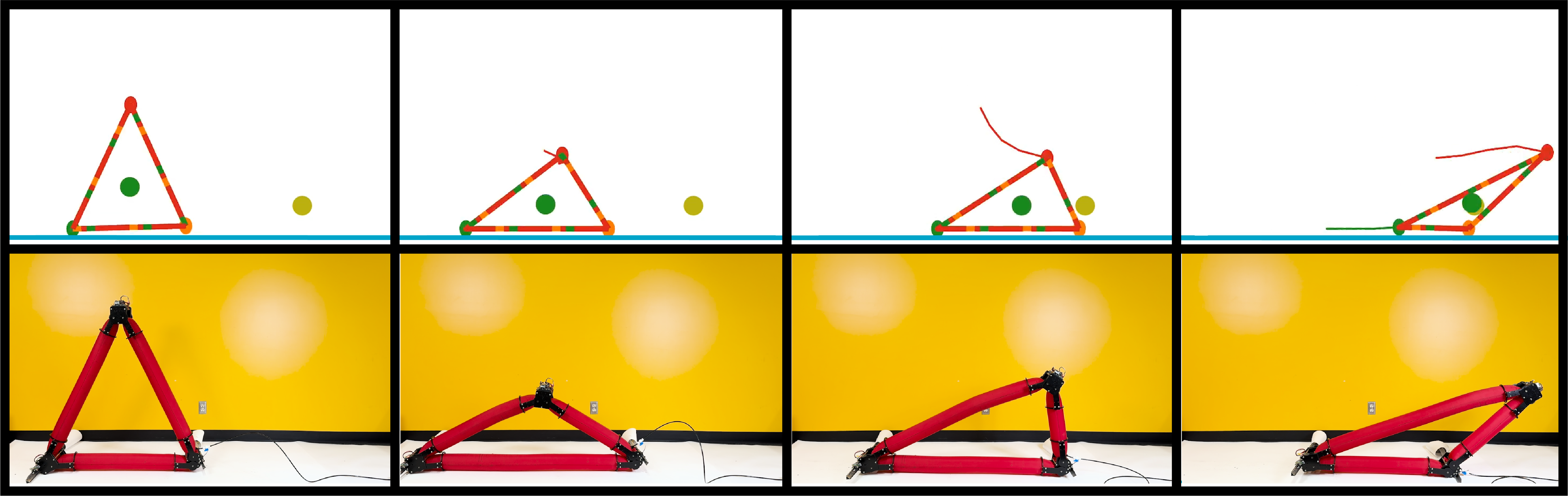}
	\caption{Deployment of DisCo on an inflatable truss robot for locomotion on a flat terrain. The top row shows the optimized simulation results, while the bottom row shows the time-matched real-world results.}
	\label{fig:truss_robot}
\end{figure}

\begin{figure}[ht]
	\centering
	\includegraphics[trim={0 0 0 15ex}, clip,width=0.95\columnwidth]{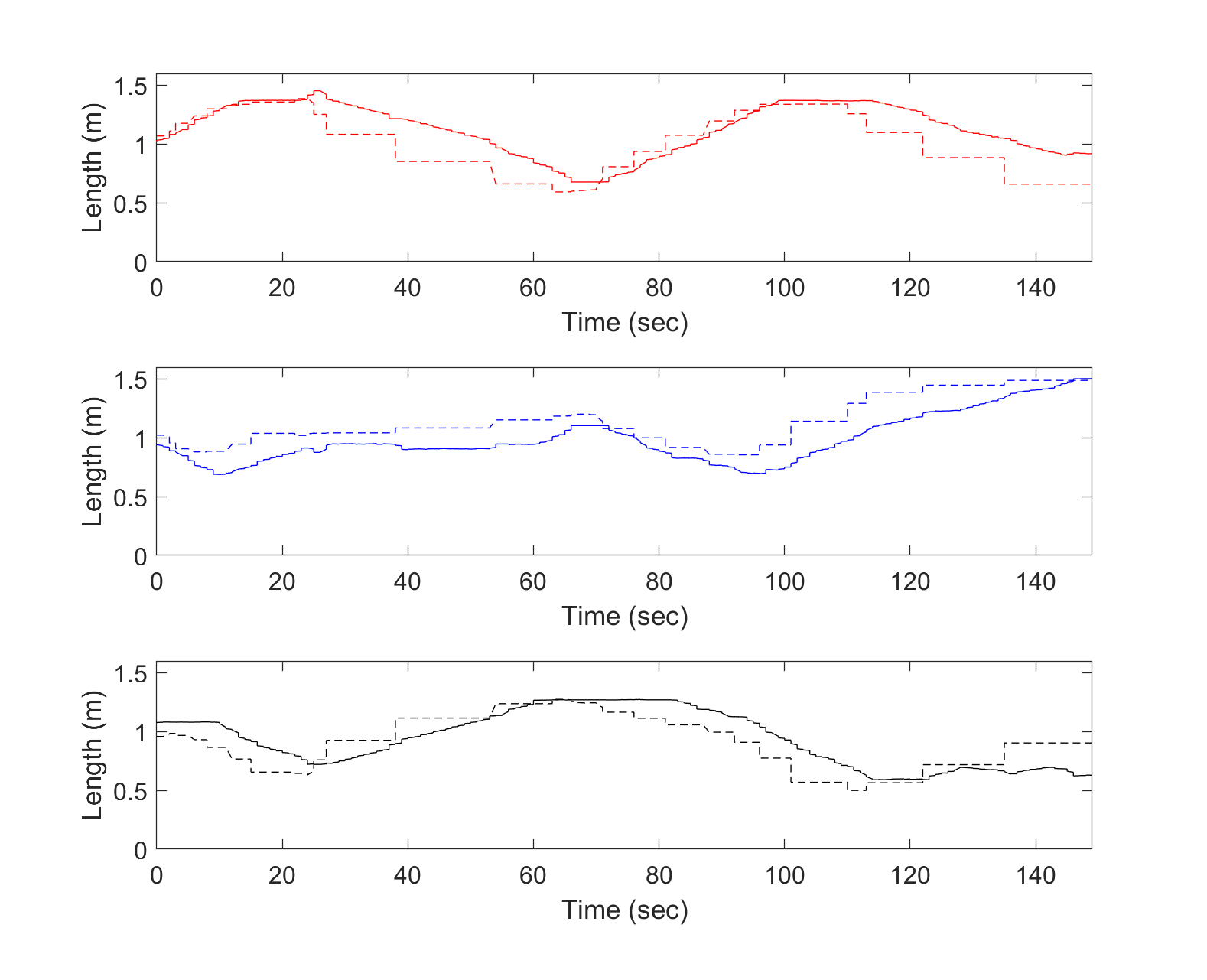}
	\caption{Measurement of the truss robot's edge lengths (solid line) compared to the optimized trajectory (dashed line), one node per row. Possible drift may have occurred from the slower high-torque motors,  or out-of-plane affects not captured from the single camera system.}
	\label{fig:truss_robot_quant}
\end{figure}

	\section{Conclusion}
\label{sec:conclusion}
We present DisCo an distributed algorithm for contact-rich trajectory optimization problems, enabling a group of robots to leverage dexterous contact interactions to collaboratively perform multi-robot tasks, such as collaborative manipulation, tactics planning in robot team sports, and modular-robot locomotion. By separating the objective functions and non-smooth dynamics constraints among the robots, each robot solves a local optimization problem only considering a single set of complementarity constraints for its contact dynamics, overcoming the notable numerical optimization challenges to solving the global contact-implicit trajectory optimization problem. Particularly, our method improves the success rates of centralized contact-implicit trajectory optimization solvers, while requiring shorter computation times. Each robot communicates with its neighbors to compute its torques and contact forces, without computing the torques and contact forces of other robots, allowing our method to scale efficiently to larger-scale problems. Future work will seek to extend our proposed method to a broader class of contact-implicit problems where decoupling the dynamics might be infeasible, while improving its computation speed and scalability.
	
	\bibliography{references}
	\bibliographystyle{base/IEEEtran}
		
\end{document}